\def\year{2021}\relax
\documentclass[letterpaper]{article} 
\usepackage{aaai21}  
\usepackage{times}  
\usepackage{helvet} 
\usepackage{courier}  
\usepackage[hyphens]{url}  
\usepackage{graphicx} 
\usepackage{natbib}  
\usepackage{caption} 

\usepackage{subfigure}
\usepackage{threeparttable}
\usepackage{bm}
\usepackage{booktabs}       
\usepackage{amsmath, amsfonts, amsthm} 
\newtheorem{theorem}{Theorem}[section]

\usepackage{xcolor}

\title{Robusta: Robust AutoML for Feature Selection via Reinforcement Learning}
\author{

    Xiaoyang Wang \textsuperscript{\rm 1},
    Bo Li \textsuperscript{\rm 1},
    Yibo Zhang \textsuperscript{\rm 1},
    Bhavya Kailkhura \textsuperscript{\rm 2},
    Klara Nahrstedt \textsuperscript{\rm 1}\\
}
\affiliations{

    \textsuperscript{\rm 1}   Department of Computer Science\\ University of Illinois at Urbana-Champaign\\

    \textsuperscript{\rm 2}   Center for Applied Scientific Computing\\ Lawrence Livermore National Laboratory\\

    xw28, lbo, jackyzhang@illinois.edu, kailkhura1@llnl.gov, klara@illinois.edu

}

\begin{document}

\maketitle

\begin{abstract}
Several AutoML approaches have been proposed to automate the machine learning (ML) process, such as searching for the ML model architectures and hyper-parameters.
However, these AutoML pipelines only focus on improving the learning accuracy of benign samples while ignoring the ML model robustness under adversarial attacks.
As ML systems are increasingly being used in a variety of mission-critical applications, improving the robustness of ML systems has become of utmost importance.
In this paper, we propose the first robust AutoML framework, {\em Robusta}--based on reinforcement learning (RL)--to perform feature selection, aiming to select features that lead to both accurate and robust ML systems.
We show that a variation of the 0-1 robust loss can be directly optimized via an RL-based combinatorial search in the feature selection scenario.
In addition, we employ heuristics to accelerate the search procedure based on feature scoring metrics, which are mutual information scores, tree-based classifiers feature importance scores, F scores, and Integrated Gradient (IG) scores, as well as their combinations.
We conduct extensive experiments and show that the proposed framework is able to improve the model robustness by up to 22\% while maintaining competitive accuracy on benign samples compared with other feature selection methods.
\end{abstract}

\section{Introduction}
Despite the remarkable success of machine learning (ML) approaches, the existence of adversarial examples has raised serious concerns regarding their utility in safety-critical applications such as banking, IoT, autonomous driving, etc. Recent works~\cite{eykholt2018robust, li2019adversarialc} show that real-world model-evasion adversarial attacks are realistic on image and audio data as well. Model-evasion adversarial examples are data samples that are added with imperceptible adversarial perturbations to cause the ML models to make wrong predictions at test/inference time. 

Recently, researchers have proposed many methods to improve the ML model robustness against adversarial examples. Most of these methods can be categorized into two major classes: adversarial training~\cite{goodfellow2018defense, madry2017towards} and robust regularization~\cite{finlay2018lipschitz}. The adversarial training generates adversarial examples during training and then uses them to augment the dataset. The robust regularization, on the other hand, adds a regularization term in the training objective to control certain properties of the ML model, e.g., smoothness. However, none of these approaches takes feature selection into consideration, which is a crucial step in the ML pipeline. Although the popularity of deep learning seemingly diminishes the importance of feature selection (or engineering), it is still a crucial component in many safety-critical scenarios, where people tend to trust simpler models.

On the feature selection side, several stable feature selection methods~\cite{khaire2019stability} have been proposed, which aim to produce consistent feature selection results under small data perturbations. The idea behind stable feature selection is to inject noise into the original dataset by generating bootstrap samples of the data and then use a base feature selection algorithm (like the LASSO) to find out which features are important in every sampled version of the data. However, in the experiment, we find that the stable feature selection algorithm indeed selects a consistent subset of features but leads to poor robustness against adversarial attacks. The reason is that the base feature selection algorithm has no idea about the feature robustness. Thus, there is no guarantee that the base algorithm will select only robust features from a pool of robust and non-robust features. In other words, the stability and robustness of feature selections algorithms may be orthogonal concepts, and improving stability does not result in improving the robustness.


So far, improving the ML model robustness against adversarial attacks in the feature selection scenario is still non-trivial due to three fundamental limitations. First, there does not exist an {\em effective metric to quantify the feature robustness} against adversarial attacks, which makes it hard to search for robust features. Second, there is no {\em general adversarial robustness evaluation framework} to be used for feature selection, because most of the existing methods are either attack/model specific or considered intractable. Third, achieving good performance and robustness requires {\em carefully making the trade-off decisions} \cite{tsipras2018robustness, zhang2019theoretically}.

In this paper, we overcome all these limitations and propose the first robust AutoML framework to perform feature selection to achieve both high performance as well as high robustness. Specifically, we present a fully automated reinforcement learning-based (RL) framework, referred to as \textit{Robusta}, to automatically search for informative and robust features. Our RL agent efficiently searches over a combinatorial space under the guidance of heuristics, which are based on feature scoring metrics, and automatically achieves a desirable trade-off between performance and robustness. We also introduce Integrated Gradient (IG)~\cite{sundararajan2017axiomatic} as a new feature scoring metric, to enable our RL agent to find robust features in the search space. Furthermore, we directly optimize an attack-agnostic 0-1 robust loss, which is considered intractable in previous works~\cite{zhang2019theoretically, zhai2020macer}. 
Compared with other multi-objective feature selection methods \cite{xue2014multi}, our method automatically performs end-to-end optimization without manually designed pipelines.

{\bf \underline{Technical Contributions}}
In this paper, we take the first step
towards developing a general robust feature selection framework. Our major contributions are as follows:

\begin{itemize}
\item Proposed an RL-based framework, called Robusta, to improve the ML model robustness at the feature level.
\item Designed a general 0-1 robust loss and embedded it into a non-sparse RL reward.
\item Adopted Integrated Gradient to guide the framework searching for robust features.
\item Examined the method with a variety of dataset subject to adversarial attacks.
\end{itemize}

\section{Preliminaries and Problem Setup}
\label{problem_setup}

\paragraph{Notations}
Given the input data $\bm{X}$ with $N$ samples and $M$ features, we table the data as $\bm{X} = \{\bm{x}_{i} \mid \bm{x}_{i} \in \mathcal{X} \subset \mathbb{R}^{M}, i = 1, ..., N\}$. In addition, we use $\bm{x}^{j}$ to denote the data associated with feature $j$. The input data $\bm{X}$ is associated with labels $\bm{Y} = \{y_{i} \mid y_{i} = \{1, 2, ..., K\}, i = 1, ..., N\}$. We assume each $(\bm{x}, y)$ is drawn from an unknown distribution $\mathcal{D}$. We use $\bm{c} \in \mathcal{C} = \{c^{j} \mid c^{j} \in \{0, 1\}, j = 1, ..., M\}$ to denote the feature selection vector. Once feature $j$ is selected, its corresponding $c^{j}$ is set to $1$ in $\bm{c}$. We use $m$ to denote the number of non-zero elements in $\bm{c}$. We define a feature selection function $g_{\bm{c}}: \mathbb{R}^{M} \xrightarrow{} \mathbb{R}^{m}$ parameterized by $\bm{c}$. We denote a subset of $\bm{X}$ which is selected by $g_{\bm{c}}$ as $\bm{Z} = \{\bm{z}_{i} \mid \bm{z}_{i} \in \mathcal{Z} \subset \mathbb{R}^{m}, g_{\bm{c}}(\bm{x}_{i}) = \bm{z}_{i}, i = 1, ..., N\}$. We use $f_{\bm{w}} : \mathbb{R}^{m} \xrightarrow{} \mathcal{Y}$ to denote an ML model parameterized by $\bm{w}$ and use $Q_{\theta} : \mathcal{S} \times \mathcal{A} \xrightarrow{} \mathbb{R}$ to denote a Q-network parameterized by $\theta$. We use $s \in \mathcal{S}$ and $a \in \mathcal{A}$ to denote the state and action in RL, respectively. We use $\mathbf{1}_{\{\mathrm{event}\}}$ to represent an indicator function, which is 1 if the event happens and 0 otherwise. We define $s_{0}$ as the initial state. We assume each $(s, a)$ pair comes from a behavior distribution $\rho (\cdot)$. Given $\mathcal{A}$, a policy $\pi \in \Pi$ over $\mathcal{S}$ is any function $\pi: \mathcal{S} \xrightarrow{} \mathcal{A}$. Reward function $R: \mathcal{S} \times \mathcal{A} \times \mathcal{S} \xrightarrow{} \mathbb{R}$ assigns reward $r$ for acting $a$ at state $s$ which lead to next state $s'$. A value function $V^{\pi}(s) = \mathbb{E}\big[r_{1} + \gamma r_{2} + \gamma^{2} r_{3} + ... \mid \pi, s \big]$ represents the discounted cumulative reward, where $r_{i}$ is the reward received on the $i^{\mathrm{th}}$ step of applying policy $\pi$ at state $s$ and $\gamma$ is the discount factor. We use $\delta$ to denote a perturbation created by adversarial attack, which can be applied to $\bm{x}$. We use $\mathbb{B}(\bm{x}, \epsilon)$ to represent a neighborhood of $\bm{x}: \{\bm{x'} \in \mathcal{X} \mid ||\bm{x'} - \bm{x}|| \leq \epsilon\}.$

\paragraph{Robust Radius} By definition, the $l_{\infty}^{\epsilon}$-robustness of $f_{\bm{w}}$ at a data point $(\bm{x}_{i}, y_{i})$ is decided by the radius of the largest $l_{\infty}$ ball centered at $\bm{x}_{i}$ in which the prediction made by $f_{\bm{w}}$ does not change. This radius is called the \textit{robust radius}, which is formally defined as

\begin{equation}
    \begin{split}
        &\phi(f_{\bm{w}}; \bm{x}, y) = \\
        &\left\{
        \begin{array}{ll} 
            &\max \{\phi \geq 0 \mid f_{\bm{w}}(\bm{x'}) = y, \forall \bm{x'} \in \mathbb{B}(\bm{x}, \phi)\}, \\ &\mathrm{when}~f_{\bm{w}}(\bm{x}) = y \\
            &0, \mathrm{otherwise}
        \end{array} 
    \right.
    \end{split}
\end{equation}

\paragraph{Robust Error} To characterize the robustness of a classifier $f_{\bm{w}}(\bm{x})$, we define \textit{robust (classification) error} using a straightforward empirical risk minimization (ERM) formulation \cite{zhang2019theoretically} 
under the threat model of bounded $\epsilon$ perturbation on data sample $(\bm{x}, y)$ as

\begin{equation}
    \ell_{\epsilon-robust}^{0-1} (f_{\bm{w}}; \bm{x}, y) := \mathbf{1}_{\{\phi(f_{\bm{w}}; \bm{x}, y) < \epsilon\}}.
\end{equation}

Besides the robust 0-1 loss, several previous works replace the 0-1 loss with surrogate losses. We discuss why we choose the 0-1 loss from a framework's perspective in Section 3.

The robust error can be decomposed as a sum of two error terms, which are named as \textit{natural error} and \textit{boundary error} in~\cite{zhang2019theoretically} because event $\phi(f_{\bm{w}}; \bm{x}, y) < \epsilon$ happens when: (1) the ML model makes a wrong prediction on $\bm{x}$ or (2) the ML model makes a correct prediction on $\bm{x}$ but the robust radius is not large enough. Thus, we have
\begin{equation}
\label{equation:robust_error_individual}
    \begin{split}
        \ell&_{\epsilon-robust}^{0-1} (f_{\bm{w}}; \bm{x}, y) \\
        &= \mathbf{1}_{\{f_{\bm{w}}(\bm{x}) \neq y\}} + \mathbf{1}_{\{(f_{\bm{w}}(\bm{x}) = y) \cap( \phi(f_{\bm{w}}; \bm{x}, y) < \epsilon)\}} \\
        &= \ell_{nat}^{0-1}(f_{\bm{w}}; \bm{x}, y) + \ell_{\epsilon-bdy}^{0-1}(f_{\bm{w}}; \bm{x}, y).
    \end{split}
\end{equation}

Although directly minimizing the the error above seems tempting, verifying the robust radius is proved to be NP-hard~\cite{weng2018towards}. Thus, we instead try to generally maximize the robustness of the ML model against adversarial examples. As recent works~\cite{ford2019adversarial, He_2019_CVPR} bridge the gap between adversarial examples and corrupted examples with additive Gaussian
noise, we adopt an attack-agnostic error to measure the robustness of an ML model against adversarial attacks. Before introducing the error, we first define $\bm{x'} \sim \mathcal{N}(\bm{x}, \sigma^{2}I)$ as the corrupted example. According to~\cite{ford2019adversarial}, we let $\sigma_{\bm{x}, \mu}$ be the $\sigma$ where the error rate $\mu = \mathop{\mathbb{E}}_{\bm{x'} \sim \mathcal{N}(\bm{x}, \sigma^{2}I)} \big[ f_{\bm{w}}(\bm{x'}) \neq y \big]$. Then, we use the \textit{expected Gaussian error}, which is shown below, to replace the expected boundary error.
\begin{equation}
    \mathcal{L}_{\epsilon-Gaussian}^{0-1}(f_{\bm{w}}) = \mathop{\mathbb{E}}_{\substack{(\bm{x}, y) \sim \mathcal{D} \\ \bm{x'} \sim \mathcal{N}(\bm{x}, \sigma_{\bm{x}, \mu}^{2}I)}} \Big[ \mathbf{1}_{\{f_{\bm{w}}(\bm{x}) = y, f_{\bm{w}}(\bm{x'}) \neq y\}} \Big].
\end{equation}

Next, we derive the estimated expected robust (classification) error from the dataset while taking feature selection vector $\bm{c}$ and feature selection function $g_{\bm{c}}$ into account. As $\bm{w_{c}}$ is decided by the training algorithm for a given dataset and a given set of features, we remove it from the parameters of the robust (classification) error. We define $\{\bm{x}'_{i,l} \mid \mathbb{E}_{\bm{w} \sim p_{\bm{w} \mid D}}\big[ f_{\bm{w}}(\bm{x}'_{i,l}) \neq y_{i}\big]\}_{l=1}^{L_{i}}$ as $L_{i}$ samples drawn from $\mathcal{N}(\bm{x}_{i}, \sigma_{\bm{x}_{i}, \mu}^{2}I)$ 
where $p_{\bm{w} \mid D}$ denotes the posterior of $\bm{w}$ given an observed dataset $D$. Condition $\mathbb{E}_{\bm{w} \sim p_{\bm{w} \mid D}}\big[ f_{\bm{w}}(\bm{x}'_{i,l}) \neq y_{i}\big]$ is added to improve the efficiency of our framework by skipping the corrupted data samples which are not harmful. The expectation is estimated by training $f_{\bm{w}}$ and evaluating $f_{\bm{w}}(\bm{x}'_{i,l})$ multiple times. In the following sections, we refer to the equation below as \textit{0-1 robust loss}

\begin{equation}
    \begin{split}
        &\hat{\mathcal{L}}_{\epsilon-robust}^{0-1} (g_{\bm{c}}) := \hat{\mathcal{L}}_{nat}^{0-1}(g_{\bm{c}}) + \hat{\mathcal{L}}_{\epsilon-Gaussian}^{0-1}(g_{\bm{c}})  \\
        &= \frac{1}{N} \sum_{i=1}^{N} \mathbf{1}_{\{f_{\bm{w_{c}}}(g_{\bm{c}}(\bm{x}_{i})) \neq y_{i}\}} + \frac{1}{\sum_{i=1}^{N} L_{i}} \sum_{i=1}^{N} \sum_{l=1}^{L_{i}} \mathbf{1}_{\{E_{adv}\}}
    \end{split}
    \label{equation:error_func}
\end{equation}
where $E_{adv} = (f_{\bm{w_{c}}}(g_{\bm{c}}(\bm{x}_{i})) = y_{i})\cap(f_{\bm{w_{c}}}(g_{\bm{c}}(\bm{x'}_{i,l})) \neq y_{i})$
\paragraph{Deep Q-Learning}
We use model-free deep Q-learning~\cite{mnih2013playing} with function-approximation to learn the action-value Q-function. For each state, $s \in \mathcal{S}$ and action, $a \in \mathcal{A}$, a Q-network is defined as:
\begin{equation}
\label{q_learning}
    Q_{\theta}(s, a) = \mathop{\mathbb{E}}_{s' \in \mathcal{S}} \Big[r_{s, a} + \gamma\max_{a'\in \mathcal{A}}Q_{\theta}(s', a') \mid s, a \Big]
\end{equation}

where $r_{s, a}$ is the reward of action $a$ in state $s$, $\gamma$ is a discount factor, $\max_{a'}Q(s', a')$ represents the maximum cumulative reward for the future states $s'$ and action $a'$. The deep Q-learning then iteratively optimizes the following loss function using gradient descent:
\begin{equation}
\label{nn_approximation}
    \mathcal{L}(\theta) = \mathop{\mathbb{E}}_{(s, a) \sim \rho(\cdot)}\Big[ \big(\mathop{\mathbb{E}}_{s' \in \mathcal{S}} \Big[R_{s, a} \mid s, a \Big] - Q_{\theta}(s, a) \big) ^2 \Big]
\end{equation}
where $R_{s, a} = r_{s, a} + \gamma\max_{a'\in \mathcal{A}}Q_{\theta_{Old}}(s', a')$ and $Q_{\theta_{Old}}$ is the target network, which gets updated periodically.
The parameter update follows standard gradient descent: $\theta_{t + 1} \xleftarrow{} \theta_{t} + \alpha \nabla_{\theta_{t}}\mathcal{L}(\theta_{t})$, where $\alpha$ is the learning rate, $t$ is the training step. When the Q-network training converges, we construct an optimal policy $\pi^{*}(s) = \mathop{\arg\max}_{a \in \mathcal{A}}Q^{*}(s, a)$, which yields highest discounted cumulative reward $V^{\pi^{*}}(s)$.
\paragraph{Problem Setup}
We want to train a single-agent deep Q-learning model that learns a policy $\pi$ which performs a sequence of actions on the feature selection vector $\bm{c}$. The policy $\pi$ produces $\bm{c}_{\pi}$ that minimizes the 0-1 robust loss defined in Equation~\eqref{equation:error_func} by maximizing $V^{\pi}(s_{0})$ of the initial state $s_{0}$.

\begin{equation}
    \bm{c}_{\pi} = \mathop{\arg\min}_{\bm{c} \in \mathcal{C}} \hat{\mathcal{L}}_{\epsilon-robust}^{0-1} (g_{\bm{c}})
\end{equation}

\begin{figure}[!t]
\begin{center}
\centerline{\includegraphics[width=1.\columnwidth]{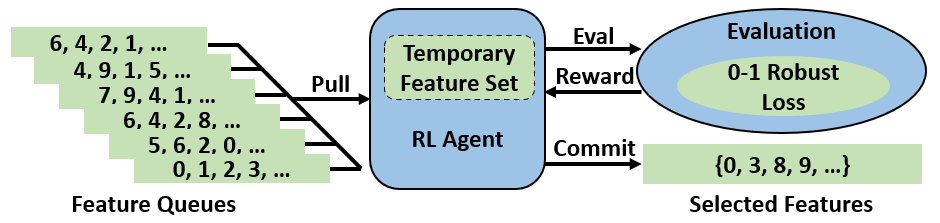}}
\caption{\textbf{RL-based Robust Features Selection Framework.}}
\label{fig:framework}
\end{center}
\end{figure}

\section{Robusta Framework Design}
A tempting approach for robust feature selection is to make the 0-1 robust loss and feature selection procedure differentiable, then, the standard gradient descent method can be employed to perform optimization. However, although differentiable surrogate loss has been designed for the 0-1 robust loss and shows success, they all introduce constraints on the downstream ML model. For example, MACER \cite{zhai2020macer} requires the ML model to be smoothed and TRADES \cite{zhang2019theoretically} only works for binary classification. As an AutoML framework, we want our proposed method to work for any downstream ML model. Moreover, there are related works \cite{nguyen2013algorithms, xue2020robust} indicate that the 0-1 loss is more robust against noise and adversarial attacks. Besides, differentiable feature selection is still problematic. Concrete distribution-based differentiable feature selection~\cite{abid2019concrete} has two major drawbacks which lead to sub-optimal solutions: (1) users need to specify the number of features $m$ to be selected, (2) the differentiable feature selection algorithm does not converge for some values of $m$. On the other hand, using RL to perform combinatorial optimization~\cite{nazari2018reinforcement, chen2019learning} and using a combinatorial search to directly optimize a 0-1 loss~\cite{nguyen2013algorithms} are shown to be successful in practice. Thus, we use an RL agent in Robusta to perform a combinatorial search in the domain $\mathcal{C}$ of the feature selection vector $\bm{c}$. We start with $\bm{c}_{0}$, in which all the elements are zero. At step $t$, the RL agent sets one element in $\bm{c}_{t-1}$ to one. The proposed framework, Robusta, is depicted in Figure~\ref{fig:framework}.

\subsection{Action}
\label{section:action}
The \textit{action}, $a \in A$ of the RL agent is defined as picking a zero element $c^{i}$, indexed by $i$ in $\bm{c}$, and set $c^{i}$ to $1$. If we use $a_{i}$ to denote picking $c^{i}$, the dimension of action space $\mathcal{A}$ equals to the number of features. Since the number of features can easily go beyond a few hundred even for a small dataset, the dimension of action space will become too large to be feasible. Thus, we utilize feature scoring metrics to craft an action space, whose dimension is invariant to the number of features. The RL agent picks a head element from one of the feature queues, which contains a list of features index ordered by their scores according to the corresponding feature scoring metric. The feature index with a higher score will be placed closer to the head of the queue. We have one queue for each feature scoring metric. We discuss the choice of metrics in section~\ref{section:metrics}.

\subsection{State}
\label{rl_state}

The \textit{state}, $s \in S$ represents the quality of selected features, the actions we have taken and the feature indexes our RL agent can pick in the next step. We use the 0-1 robust loss, which is obtained by training an ML model, to measure the quality of selected features. We use \textit{Action Counter} and \textit{Queue State} to represent previous actions as well as possible next steps, respectively. Action counter counts how many times an action has been taken. Queue state records the feature scores indexed by the head elements in each feature queue. By combining these three pieces of information, we provide the RL agent with an intuition about how promising an action can be in the current state.

\subsection{Reward}
\label{section:reward}

When the RL agent is about to terminate at step $t = T$ after applying a sequence of actions on $\bm{c}_{0}$ and produces $\bm{c}_{t}$, we give the agent a reward which is $(1 - \hat{\mathcal{L}}_{\epsilon-robust}^{0-1} (g_{\bm{c}_{t}})) \times 100$. To make the notation clear, we only use $1 - \hat{\mathcal{L}}_{\epsilon-robust}^{0-1} (g_{\bm{c}_{t}})$ in the following sections. The reward function is defined as
\begin{equation}
\label{equation:raw_reward}
    R(s_{t}, a_{t}, s_{t+1}) = \left\{
        \begin{array}{cl}
            1 - \hat{\mathcal{L}}_{\epsilon-robust}^{0-1} (g_{\bm{c}_{t}}) &, s_{t+1} \mathrm{\ is\ terminal\ }\\ 
            0 &, \mathrm{otherwise.}
        \end{array} 
    \right.
\end{equation}

We discuss the terminate condition in next section. If we only give the RL agent reward when it terminates, the agent will receive zero rewards for most of the steps. In other words, the reward would be sparse. The sparse reward usually makes the RL agent learning slow based on the regular RL principle. To address this sparse reward issue, we design a reward shaping function.

\paragraph{Reward Shaping} Reward shaping~\cite{ng1999policy} is a method used in reinforcement learning whereby additional training rewards are provided to guide the learning agent. We define reward shaping function in Robusta as

\begin{equation}
\label{equation:reward_shaping_func}
    \begin{split}
        &F(s, a, s') = \gamma\Phi(s') - \Phi(s) \\
        &= \gamma\big( 1 - \hat{\mathcal{L}}_{\epsilon-robust}^{0-1} (g_{\bm{c}_{s'}}) \big) - \big( 1 - \hat{\mathcal{L}}_{\epsilon-robust}^{0-1} (g_{\bm{c}_{s}}) \big) \\
    \end{split}
\end{equation}

where $\bm{c}_{s}$ and $\bm{c}_{s'}$ are the feature selection vectors in state $s$ and its following state $s'$, respectively. With this reward shaping function, we define our shaped reward function $R'(s, a) = R(s, a) + F(s, a)$. Using $R'$ yields faster learning speed because the RL agent will receive a non-zero reward at most of the steps. Then, we show the reward shaping function in Equation~\eqref{equation:reward_shaping_func} will not affect the optimal policy $\pi^{*}$ we want to learn.

\begin{theorem}
\label{theorem:3.1} (Proof in Appendix~\ref{proof:theorem3.1})
    Let us assume $\mathcal{E} = (\mathcal{S}, \mathcal{A}, \gamma, R)$ and $\mathcal{E'} = (\mathcal{S}, \mathcal{A}, \gamma, R')$ are two identical environments except reward function. If a reward function $R: \mathcal{S} \times \mathcal{A} \times \mathcal{S} \xrightarrow{} \mathbb{R}$ and a reward shaping function $F: \mathcal{S} \times \mathcal{A} \times \mathcal{S} \xrightarrow{} \mathbb{R}$ satisfy $\sum_{t=1}^{T} F(s_{t}, \pi(s_{t}), s_{t+1}) = \sum_{t=1}^{T} R(s_{t}, \pi(s_{t}), s_{t+1}) = R(s_{T}, \pi(s_{T}), s_{T+1})$ for any policy $\pi$ and any sequence length $T$, 
    then, we have $\pi_{\mathcal{E}}^{*} = \pi_{\mathcal{E'}}^{*}$ when $\gamma = 1$.
\end{theorem}

It is easy to show that our $R$ and $F$ satisfy the requirements in Theorem~\ref{theorem:3.1} because the changes of 0-1 robust loss at each state add up to the final 0-1 robust loss.

\subsection{Implementation Details}
\label{section:implementation_details}

We discuss two rules used in our implementation.

\paragraph{Terminate Condition} For a given sequence of state $\{s_{0} \xrightarrow{} s_{1} \xrightarrow{} s_{2} \xrightarrow{} ... \xrightarrow{} s_{t} \xrightarrow{} ...\}$, the state $s_{t+1}$ is marked as terminal state if $\hat{\mathcal{L}}_{\epsilon-robust}^{0-1} (g_{\bm{c}_{s_{t+k}}}) - \hat{\mathcal{L}}_{\epsilon-robust}^{0-1} (g_{\bm{c}_{s_{t}}}) \geq 0$, which means the RL agent makes no progress in the following $k$ states of $s_{t}$. In the implementation, we set $k$ to 5. Once the terminal state $s_{t+1}$ is identified, its following states get discarded.

\paragraph{Eliminate Irrelevant and Non-robust Features}
Since the actions defined in section~\ref{section:action} only select features and add them to the selected subset, they unavoidably select bad features at some steps because we have no way to skip them. Thus, we directly eliminate a feature $i$ selected by action $a_{t}$ at step $t$ from the environment if $\hat{\mathcal{L}}_{\epsilon-robust}^{0-1} (g_{\bm{c}_{t-1}}) - \hat{\mathcal{L}}_{\epsilon-robust}^{0-1}(g_{\bm{c}_{t}}) \geq 0.05$. In addition, we delete $a_{t}$ and $s_{t}$ and again let the RL agent start at state $s_{t-1}$. Through out the experiment, the RL agent drops 66.6\%, 2.4\%, 0.2\% and 4.5\% features on average from SpamBase, Isolet, MNIST and CIFAR10 dataset, respectively. It drops a majority of features in SpamBase because most of the features are non-robust and the agent ends up with selecting only 4 features.

\section{Feature Selection Metrics in Robusta}
\label{section:metrics}

As we have mentioned in section~\ref{section:action}, we provide multiple feature scoring metrics for the RL agent. This approach falls naturally into the ensemble method, where output is obtained by aggregating the outputs from a collection of weak single models. Ensemble method has been well-studied and is shown to be effective in classification/regression tasks and feature selection task~\cite{saeys2008robust}. Since we do not have any reliable metric to scoring the features based on their robustness, as is shown in Appendix~\ref{appendix:metrics}, the ensemble method provides more potentials for selecting a subset of robust features by combining relevant scoring metric. In this work, we employ three widely used metrics in Robusta to score features based on their predictive power and propose one metric to score features based on their robustness against adversarial attacks.

\subsection{Performance Metric}
There are many performance metrics in feature selection based on information theory, redundancy and relevance \cite{peng2005feature} and, etc. In this work, we mainly aim to develop a general framework for wrapper-based ensemble feature selection. Thus, we only choose three simple and well-known performance metrics to show the effectiveness of Robusta.

\paragraph{Mutual Information Score} Mutual information (MI) score, $MI(X^{j}, Y)$, is a measure between two random variables $X^{j}$ and $Y$, where $X^{j} \sim D_{\bm{x}^{j}}$ and $D_{\bm{x}^{j}}$ denotes the data distribution associated with $\bm{x}^{j}$. MI score quantifies the amount of information obtained about one random variable, through the other random variable. The higher the mutual information score is, the more predictive power feature $j$ has for predicting $Y$. 

\paragraph{Tree Score} Tree score is the importance score where a trained tree classifier assigns to each feature. This metric leverages the intrinsic properties of the tree-based classifier.

\paragraph{F Score} $F_{score}$ is a transformation of $F_{value}$. $F_{value}$ is the ratio of two $\chi$-distributions divided by their degrees of freedom $N$ as equation~\eqref{f_value} shows:

\begin{equation}
\label{f_value}
    F_{value} = (\chi_{X^{i}}^{2} / N - 1) / (\chi_{Y}^{2} / N - 1)
\end{equation}

where $X^{j} \sim D_{\bm{x}^{j}}$. We convert $F_{value}$ to $F_{score}$ by taking the absolute value of $F_{value}$, normalizing, and subtracting the values from 1. Thus, a larger $F_{score}$ implies more predictive power.

\subsection{Robustness Metric}
\label{robust_metric}
To measure the robustness, we adopt Integrated Gradient (IG)~\cite{sundararajan2017axiomatic}, which is widely used in attributing the prediction $y$ of a neural network to its $j^{th}$ input features $x^{j}$, in Robusta

\paragraph{Integrated Gradient} Integrated Gradient calculates the integral of gradients w.r.t. each input feature along the path from a given baseline to the input. Formally, it can be described as follows:
\begin{equation}
\label{ig}
	\mathrm{IG}_{f_{\bm{w}}}^{j}(\Bar{\bm{x}}, \bm{x}, y) = (\Bar{x}^{j}-x^{j})\int_{\alpha=0}^{1} \frac{\partial f_{\bm{w}}^{y}(x+\alpha\times(\Bar{x}-x))}{\partial x^{j}}d\alpha
\end{equation}

where $\bm{x}$ denotes the input, $\Bar{\bm{x}} = \frac{1}{N}\sum_{i=1}^{N} \bm{x}_{j}$ denotes the baseline input. $x^{j}$ is the $j^{th}$ feature of $\bm{x}$. $f_{\bm{w}}^{y}$ denotes the ML model function associated with a prediction $y$. The IG score represents the contribution of each feature $x^{j}$ to prediction $y$. Then, we show the relationship between IG and adversarial robustness against adversarial perturbation $\delta$.

\begin{theorem} (Theorem 5.1 in \citealt{chalasani2018concise})
    If a loss function $\ell(f_{\bm{w}}; \bm{x}, y)$ is convex, we have
    \begin{equation}
    \begin{split}
        &\ell(f_{\bm{w}}; \bm{x}, y) + \max_{||\bm{x} - \bm{x'}||_{\infty} \leq \epsilon} ||\mathrm{IG}_{f_{\bm{w}}} (\bm{x}, \bm{x} + \delta, y)||_{1} \\
        &= \max_{||\delta||_{\infty} \leq \epsilon} \ell(f_{\bm{w}}; \bm{x} + \delta, y)
    \end{split}
    \end{equation}
\end{theorem}

Several popular loss functions are convex, such as Logistic Negative Log Likelihood Loss and Hinge loss. This theorem implies that improving IG stability is equivalent to $\ell_{\infty}^{\epsilon}$-adversarial training~\cite{madry2017towards}. An intuitive explanation is provided in Appendix~\ref{section:intuition_IG}. Thus, we use $\frac{1}{N} \sum_{i=1}^{N}\mathrm{IG}_{f_{\bm{w}}}(\bm{x}_{i}, \bm{x}_{i} + \delta_{i}, y_{i})$ to construct the robust feature score vector.

\paragraph{Usage} We obtain adversarial perturbation $\delta$ by directly performing adversarial attack on $f_{\bm{w}}$ and $\bm{x}$. We subtract normalized $\frac{1}{N} \sum_{i=1}^{N}\mathrm{IG}_{f_{\bm{w}}}(\bm{x}_{i}, \bm{x}_{i} + \delta_{i}, y_{i})$ from 1 to make sure the higher score indicates more robustness. We then compute the average of each performance metric score and the IG score, to generate three robust features scores. We do not use IG score independently because it does not incorporate the loss function $\ell(f_{\bm{w}}; \bm{x}, y)$.

\section{Experimental Evaluation}

We conducted extensive experiments to show the effectiveness of our proposed framework on minimizing the 0-1 robust loss $\hat{\mathcal{L}}_{\epsilon-robust}^{0-1} (g_{\bm{c}})$ by optimizing $\bm{c}$. In other words, our framework improves the \textit{robustness}, which refers to the classification accuracy of an ML model on adversarial examples, while preserves the \textit{performance}, which refers to the classification accuracy of an ML model on benign examples. We report the detailed experiment setup and in Appendix~\ref{section:experiment_setup}.

\subsection{Attack Mode}
We focus on transferable adversarial attacks\cite{goodfellow2018defense}. The adversarial attack is performed by solving an optimization problem which is shown as follow

\begin{equation}
\label{gen_adv_samples}
\begin{split}
    \bm{x}_{adv} = \mathop{\arg\max}_{\bm{x}_{adv} \in \mathbb{B}(\bm{x}, \epsilon)} \ell(f_{\bm{w}}; \bm{x}_{adv}, y)
\end{split}
\end{equation}

Then, we have the adversarial perturbation $\delta = \bm{x}_{adv} - \bm{x}$. Transferable adversarial attacks assume the true $f_{\bm{w}}$ is not available for the adversary. Thus, the adversary generates adversarial examples using $f_{\bm{w'}}$, which is similar to $f_{\bm{w}}$. In this experiment, we use two ML models $f_{\bm{w}}$ and $f_{\bm{w'}}$, which are trained independently using the same model, data and hyper-parameters. Transferable attack works based on the observation that adversarial examples that fool one model are likely to
fool another~\cite{goodfellow2018defense}. More specifically, we employ the Fast Gradient Sign Method (FGSM) adversary and the Projected Gradient Descent (PGD)~\cite{madry2017towards} adversary to test our framework.

\subsection{Datasets}

We performed experiments on four real-world datasets: SpamBase, Isolet, MNIST, CIFAR, which contain text data, audio data, image data and image data, respectively. The detailed descriptions of the datasets and the data sample allocations are reported in Appendix~\ref{section:experiment_setup}.

\subsection{Evaluating the Effectiveness of Integrated Gradient in Identifying Non-robust Features}
\label{ig_eval}
Before evaluating the effectiveness of our framework, we first design an experiment to show minimizing $\mathop{\mathbb{E}}_{(\bm{x}, y) \sim D} \big[||\mathrm{IG}_{f_{\bm{w}}} (\bm{x}, \bm{x} + \delta, y)||\big]$ helps minimizing $\mathop{\mathbb{E}}_{(\bm{x}, y) \sim D} \big[\ell(f_{\bm{w}}; \bm{x} + \delta, y)\big]$. This experiment demonstrates the effectiveness of the robust measure, which is proposed in Section \ref{robust_metric}, in Robusta. The adversary uses $f_{\bm{w'}}$ to generate adversarial examples $\bm{X}_{adv}$, which is applied on $f_{\bm{w}}$.

\begin{figure}
\centering     
\subfigure[Highest IG Score]{\label{fig:ig_performance_highest}\includegraphics[width=.153\textwidth]{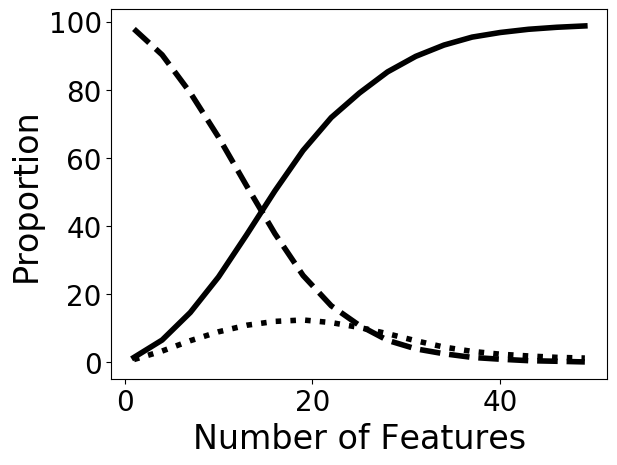}}
\subfigure[Random IG Score]{\label{fig:ig_performance_random}\includegraphics[width=.153\textwidth]{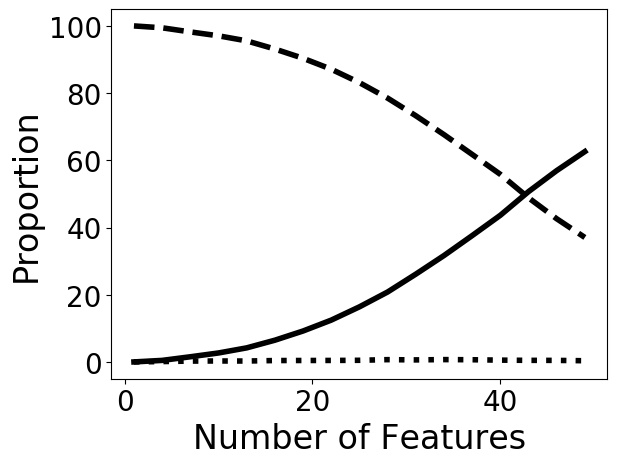}}
\subfigure[Lowest IG Score]{\label{fig:ig_performance_lowest}\includegraphics[width=.153\textwidth]{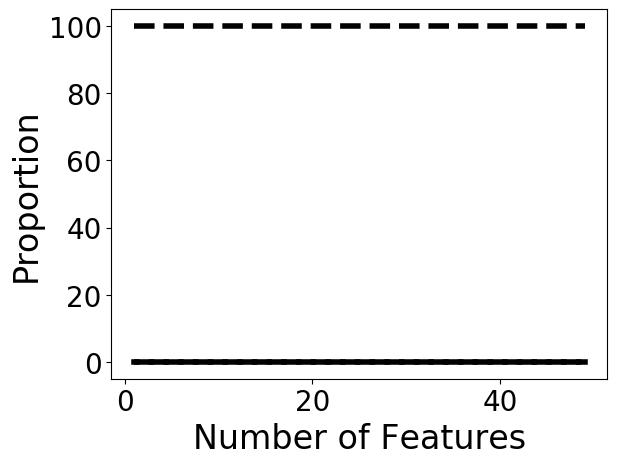}}
\caption{\textbf{Proportion of Adversarial Examples Becomes Benign After Top-$K$ Perturbations Removal}. The proportion of MNIST adversarial examples becomes benign~(solid line), the same adversarial example~(dash line), a new adversarial example~(dot line) by removing adversarial perturbations from a subset of features with highest, random and lowest IG scores, respectively. The same adversarial example refers to the adversarial example, after removing a subset of adversarial perturbations, that is assigned the same wrong label as before removing the perturbations. A new adversarial example refers to the adversarial example, after removing a subset of adversarial perturbations, that is assigned a different wrong label by the ML model.}
\label{fig:ig_performance}
\end{figure}

We reuse the notations $\bm{c}$ and $g_{\bm{c}}(\bm{x}_{0}, \bm{x}_{1})$ as element selection vector and element selection function, respectively. We define $\bm{x'} = g_{\bm{c}}(\bm{x}_{0}, \bm{x}_{1})$, where $x'^{j} = x_{0}^{j}$ if $c^{i} = 0$ and $x'^{j} = x_{1}^{j}$ if $c^{j} = 1$. With pre-computed adversarial examples and $\bm{IG} = \frac{1}{N} \sum_{i=1}^{N}\mathrm{IG}_{f_{\bm{w}}}(\bm{x}_{i}, \bm{x}_{i} + \delta, y_{i})$, we start to set elements in $\bm{c}$ to 1, which are initially set to 0. For a given $\bm{IG}$ and const number $K$, we set $c^{j}$ to 1 if $IG^{j}$ is among the top-$K$ largest elements in $\bm{IG}$. We want to see that, as $K$ increases, $\mathop{\mathbb{E}}_{(\bm{x}, y) \sim D} \big[\bm{1}_{\{f_{\bm{w}}(g_{\bm{c}_{K}}(\bm{x}, \bm{x} + \delta)) \neq y\}}\big]$ decreases. In other words, as more adversarial perturbation $\delta^{j}$ get removed from $\bm{x}$, more $\bm{x}_{adv}$ become benign. Obviously, lower $\mathop{\mathbb{E}}_{(\bm{x}, y) \sim D} \big[\bm{1}_{\{f_{\bm{w}}(g_{\bm{c}_{K}}(\bm{x}, \bm{x} + \delta)) \neq y\}}\big]$ implied lower $\mathop{\mathbb{E}}_{(\bm{x}, y) \sim D} \big[\ell(f_{\bm{w}}; \bm{x} + \delta, y)\big]$. In Figure~\ref{fig:ig_performance_highest}, we found that if the adversarial perturbation from the top 40 features with highest IG attribution is removed, $\sim$90\% adversarial examples become benign. In comparison, as Figure~\ref{fig:ig_performance_random} shows, if we randomly remove adversarial perturbation from the same amount of features, only $\sim$50\% adversarial examples become benign. Furthermore, if we remove the adversarial perturbation from the features with the least IG score, none of the adversarial examples became benign, as figure~\ref{fig:ig_performance_lowest} shows.

\subsection{Evaluating the Effectiveness of Robusta Framework}
\label{section:rl_eval}

In this section, we evaluate the effectiveness of Robusta by measuring the performance as well as the robustness of the features we selected. We also compare with the results obtained by Stable Feature Selection\footnote{\url{https://github.com/scikit-learn-contrib/stability-selection}}, LASSO and Concrete Auto-encoder (CAE) ~\cite{abid2019concrete}, which is a recently proposed SOTA method. The implementation details are reported in Appendix~\ref{section:experiment_setup}. In this experiment, we use two test sets which contain benign samples and pre-computed adversarial examples, respectively.

\paragraph{Effectiveness}
\begin{figure}
\centering     
\subfigure[Performance]{\label{fig:rl_performance}\includegraphics[width=.23\textwidth]{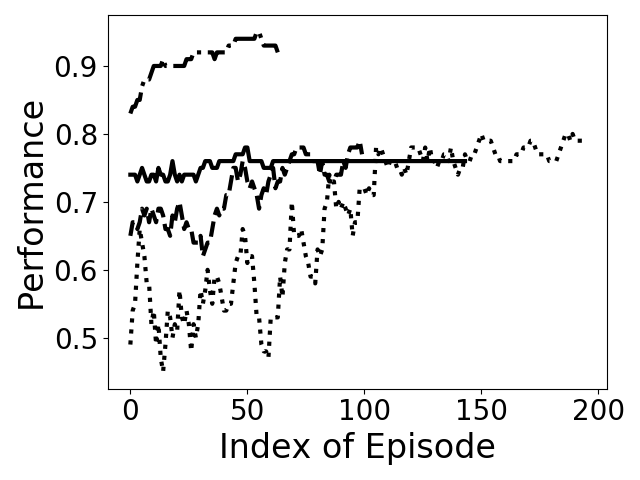}}
\subfigure[Robustness]{\label{fig:rl_robustness}\includegraphics[width=.23\textwidth]{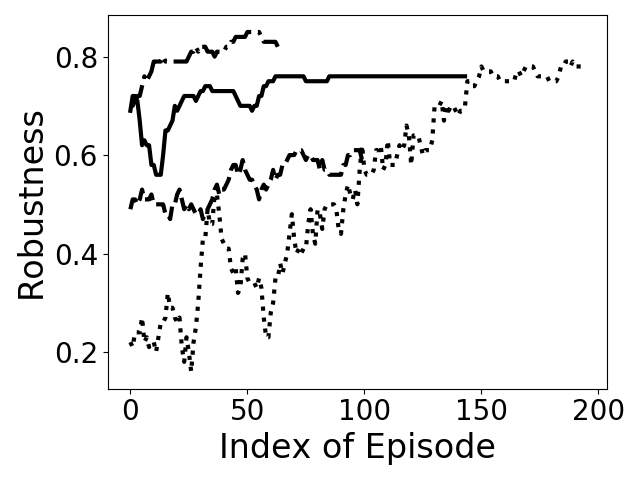}}
\caption{\textbf{Robusta Evaluation} Performance and robustness improvement on datasets SpamBase~(solid line), Isolet~(dash line), MNIST~(dash-dot line) and CIFAR10~(dot line) during training.}
\label{fig:rl_perf}
\end{figure}

We first show the performance and robustness of the select features on the test sets during RL training in Figure~\ref{fig:rl_perf}. We plot their average values over the consecutive 10 steps. The more episodes the RL agent gets trained, the better performance and robustness its selected features yield. The length of the lines varies because the steps to complete an episode differ. As we can see, the improvement varies from $5\%$ to $50\%$ across the different datasets. The amounts of fluctuation on SpamBase, Isolet and MNIST are smaller than CIFAR10. The reason is that we use the hidden representation, extracted by vanilla trained non-robust Resnet-18 and an adversarially trained robust Resnet-18, as features in CIFAR10. The differences between features are magnified by the robust and non-robust Resnet-18 networks. Thus, our CIFAR10 dataset has more diverse features. Besides, diversity also increases the difficulty of improving the performance and the robustness at the same time. As we can see from Figure~\ref{fig:rl_perf}, the line represents the CIFAR10 dataset starts from a very low value and grows slowly.







\begin{table}[!t]
\caption{Performance (accuracy on benign samples) of the ML Model using selected features}
\label{feature_performance}
\begin{center}
\begin{small}
\begin{threeparttable}
\begin{sc}
\begin{tabular}{lcccccc}
\toprule
Data set ($\epsilon$) & Stable & LASSO & Concrete & Robusta \\
\midrule
Spam ($8/255$) & \textbf{91.7} & 80.06\% & 80.36\% & 77.27\% \\

Isolet ($1/10$) & \textbf{91.7} & 76.65\% & 81.54\% & 81.99\% \\

MNIST ($1/10$) & / & 94.55\% & 97.21\% & 95.76\% \\

MNIST ($2/10$) & / & 94.54\% & 97.24\% & 95.71\% \\

MNIST ($3/10$) & / & 94.58\% & 97.22\% & 95.68\% \\

CIFAR ($8/255$) & / & 94.43\% & 94.44\% & 90.92\% \\
\bottomrule
\end{tabular}
\end{sc}
\begin{tablenotes}
\item [*]We bold the numbers if the best method outperforms all the others by 3\%.
\end{tablenotes}
\end{threeparttable}
\end{small}
\end{center}
\end{table}







\begin{table}[!t]
\caption{Robustness (accuracy on adversarial examples) of the ML model using selected features under PGD attack}
\label{feature_robustness_pgd}
\begin{center}
\begin{small}
\begin{threeparttable}
\begin{sc}
\begin{tabular}{lcccccc}
\toprule
Data set ($\epsilon$) & Stable & LASSO & Concrete & Robusta \\
\midrule
Spam ($8/255$) & 18.10\% & 55.36\% & 49.73\% & \textbf{68.03\%} \\

Isolet ($1/10$) & 25.98\% & 42.74\% & 24.13\% & \textbf{48.02\%} \\

MNIST ($1/10$) & / & 77.82\% & 77.93\% & \textbf{83.19\%} \\

MNIST ($2/10$) & / & 38.27\% & 27.10\% & \textbf{44.87\%}  \\

MNIST ($3/10$) & / & 14.14\% & 4.67\% & \textbf{18.11\%} \\

CIFAR ($8/255$) & / & 7.25\% & 14.29\% & \textbf{36.74\%} \\
\bottomrule
\end{tabular}
\end{sc}
\begin{tablenotes}
\item [*]We bold the numbers if the best method outperforms all the others by 3\%.
\end{tablenotes}
\end{threeparttable}
\end{small}
\end{center}
\end{table}


\begin{table}[!t]
\caption{Average accuracy on benign and adversarial examples of the ML model using selected features.}
\label{table:avg_acc}
\begin{center}
\begin{small}
\begin{threeparttable}
\begin{sc}
\begin{tabular}{lcccccr}
\toprule
Data set ($\epsilon$) & Stable & LASSO & Concrete & Robusta \\
\midrule
Spam($8/255$) & 54.90\% & 67.71\% & 65.05\% & \textbf{72.65\%} \\
Isolet ($1/10$) & 59.50\% & 59.70\% & 52.84\% & \textbf{65.01\%} \\
MNIST ($1/10$) & / & 41.29\% & 87.57\% & \textbf{89.48\%} \\
MNIST ($2/10$) & / & 35.55\% & 62.17\% & \textbf{70.29\%} \\
MNIS($3/10$) & / & 32.58\% & 50.95\% & \textbf{56.90\%} \\
CIFAR($8/255$)       & / & 50.84\% & 54.37\% & \textbf{63.83\%} \\
\bottomrule
\end{tabular}
\end{sc}
\begin{tablenotes}
\item [*]We bold the numbers if the best method outperforms all the others by 3\%.
\end{tablenotes}
\end{threeparttable}
\end{small}
\end{center}
\end{table}


\begin{table}[!t]
\caption{Trade-off ratio between performance and robustness of the ML model using selected features.}
\label{table:tradeoff_ratio}
\begin{center}
\begin{small}
\begin{threeparttable}
\begin{sc}
\begin{tabular}{lccccccr}
\toprule
Dataset ($\epsilon$) & Stable & LASSO & Concrete & Robusta \\
\midrule
Spam ($8/255$) & 5.07 & 1.45 & 1.62 & \textbf{1.13} \\
Isolet ($1/10$) & 3.58 & 1.79 & 3.38 & \textbf{1.71} \\
MNIST ($1/10$) & / & 1.21 & 1.24 & \textbf{1.15} \\
MNIST ($2/10$) & / & 2.47 & 3.60 & \textbf{2.13} \\
MNIST ($3/10$) & / & 6.68 & 20.82 & \textbf{5.28} \\
CIFAR ($8/255$) & / & 13.02 & 6.61 & \textbf{2.47} \\
\bottomrule
\end{tabular}
\end{sc}
\begin{tablenotes}
\item [*]The closer to 1.0, the better.
\end{tablenotes}
\end{threeparttable}
\end{small}
\end{center}
\end{table}

\paragraph{Performance and Robustness}
We then collect subsets of features, discovered by Robusta, to measure the performance and robustness of feature sets. To mitigate uncertainty, we run each method three times and collect the best features we find. Due to the limited space, we only report the average of the result in this section. We further report the variations at the end of Appendix. For Robusta, we collect the best subset of features we find throughout the RL training procedures. Since none of the competitors can decide the optimal amount of features, we manually explore the optimal number of features for each competitor. Finally, we collect \textit{around} 4, 40, 71, 30 features from each dataset, respectively. The details about deciding the number of selected features are described in Appendix~\ref{section:experiment_setup}. For each subset of features, we train a 2-layer neural network three times and report the mean of the performance and the robustness against PGD attacks in Tables~\ref{feature_performance},~\ref{feature_robustness_pgd}. We further report the robustness against FGSM attacks in Appendix~\ref{section:FGSM_robustness}. Since the stable feature selection algorithm did not converge to a result on MNIST and CIFAR10 dataset after 12 hours of running, we leave the corresponding entries blank. Our method always achieves the best robustness, whose improvement can be up to 22\% compared with the second-best reported robustness. Our framework also maintains competitive performance compared to other methods, which maintains reasonable robustness. However, the Stable feature selection delivers better performance but hurt the robustness a lot, makes it hard to compare and draw a conclusion. Thus, we report the average of performance and robustness in Table~\ref{table:avg_acc} and the trade-off ratio (the accuracy on benign samples divided by the accuracy on adversarial examples) between performance and robustness (against PGD attacks) in Table~\ref{table:tradeoff_ratio}. These two tables demonstrate that our method is always the best. With the experiment above, we show that our method yields more accuracy and makes better trade-off decisions.

\section{Related Work}

\paragraph{Automated Feature Engineering}
\citealt{khurana2018feature} and ~\citealt{liu2019automating} proposed reinforcement learning-based frameworks to efficiently explore the search space of feature transformations and subsets respectively. ~\citealt{abid2019concrete} selected feature via back-propagation by introducing Concrete Variable to an auto-encoder. However, the robustness against adversarial examples is generally ignored by these frameworks.

\paragraph{Robust Feature Engineering}
\citealt{xu2017feature} proposed a method to detect adversarial examples via feature squeezing and \citealt{tong2019improving} improved ML classifier robustness by identifying and utilizing conserved features. Nevertheless, their methods only worked for a specific type of task, which was image classification and malware detection, respectively, because they relied on domain-specific methods like Image quantization or security sandbox test. \citealt{ilyas2019adversarial} showed robust features generated by the adversarially trained neural network can improve ML model robustness. However, they didn't propose a systematical method to automatically select robust features. 

\section{Conclusion}
In this paper, we present \textit{Robusta}, an automated framework for robust feature selection, which can automatically search for a subset of informative and robust features. We leverage reinforcement learning to efficiently explore a combinatorial search space of feature sets. Also, we introduce an efficient metric, based on Integrated Gradient, to enable the RL agent to discover robust features. We design a reward function--based on a variation of the 0-1 robust loss~\cite{zhang2019theoretically}--to provide RL agent a comprehensive understanding of the quality of each feature to be selected. The experimental results show it is possible to improve ML model robustness by up to 22\% via simply selecting more robust features. Furthermore, we show that robustness improvement does not come at the expense of significant performance degradation.

\section*{Acknowledgements}
This research was funded by the National Science Foundation under the grant, NRT-HDR: Data and Informatics Graduate Intern-traineeship: Materials at the Atomic Scale (DIGI-MAT) \#1922758


\bibliography{ref}

\begin{thebibliography}{30}
\providecommand{\natexlab}[1]{#1}
\providecommand{\url}[1]{\texttt{#1}}
\providecommand{\urlprefix}{URL }
\expandafter\ifx\csname urlstyle\endcsname\relax
  \providecommand{\doi}[1]{doi:\discretionary{}{}{}#1}\else
  \providecommand{\doi}{doi:\discretionary{}{}{}\begingroup
  \urlstyle{rm}\Url}\fi

\bibitem[{Abid, Balin, and Zou(2019)}]{abid2019concrete}
Abid, A.; Balin, M.~F.; and Zou, J. 2019.
\newblock Concrete autoencoders for differentiable feature selection and
  reconstruction.
\newblock \emph{arXiv preprint arXiv:1901.09346} .

\bibitem[{Chalasani et~al.(2018)Chalasani, Chen, Chowdhury, Jha, and
  Wu}]{chalasani2018concise}
Chalasani, P.; Chen, J.; Chowdhury, A.~R.; Jha, S.; and Wu, X. 2018.
\newblock Concise explanations of neural networks using adversarial training.
\newblock \emph{arXiv} arXiv--1810.

\bibitem[{Chen and Tian(2019)}]{chen2019learning}
Chen, X.; and Tian, Y. 2019.
\newblock Learning to perform local rewriting for combinatorial optimization.
\newblock In \emph{Advances in Neural Information Processing Systems},
  6278--6289.

\bibitem[{Eykholt et~al.(2018)Eykholt, Evtimov, Fernandes, Li, Rahmati, Xiao,
  Prakash, Kohno, and Song}]{eykholt2018robust}
Eykholt, K.; Evtimov, I.; Fernandes, E.; Li, B.; Rahmati, A.; Xiao, C.;
  Prakash, A.; Kohno, T.; and Song, D. 2018.
\newblock Robust physical-world attacks on deep learning visual classification.
\newblock In \emph{Proceedings of the IEEE Conference on Computer Vision and
  Pattern Recognition}, 1625--1634.

\bibitem[{Finlay et~al.(2018)Finlay, Calder, Abbasi, and
  Oberman}]{finlay2018lipschitz}
Finlay, C.; Calder, J.; Abbasi, B.; and Oberman, A. 2018.
\newblock Lipschitz regularized deep neural networks generalize and are
  adversarially robust.
\newblock \emph{arXiv preprint arXiv:1808.09540} .

\bibitem[{Ford et~al.(2019)Ford, Gilmer, Carlini, and
  Cubuk}]{ford2019adversarial}
Ford, N.; Gilmer, J.; Carlini, N.; and Cubuk, D. 2019.
\newblock Adversarial examples are a natural consequence of test error in
  noise.
\newblock \emph{arXiv preprint arXiv:1901.10513} .

\bibitem[{Goodfellow(2018)}]{goodfellow2018defense}
Goodfellow, I. 2018.
\newblock Defense against the dark arts: An overview of adversarial example
  security research and future research directions.
\newblock \emph{arXiv preprint arXiv:180604169} .

\bibitem[{He, Rakin, and Fan(2019)}]{He_2019_CVPR}
He, Z.; Rakin, A.~S.; and Fan, D. 2019.
\newblock Parametric Noise Injection: Trainable Randomness to Improve Deep
  Neural Network Robustness Against Adversarial Attack.
\newblock In \emph{The IEEE Conference on Computer Vision and Pattern
  Recognition (CVPR)}.

\bibitem[{Ilyas et~al.(2019)Ilyas, Santurkar, Tsipras, Engstrom, Tran, and
  Madry}]{ilyas2019adversarial}
Ilyas, A.; Santurkar, S.; Tsipras, D.; Engstrom, L.; Tran, B.; and Madry, A.
  2019.
\newblock Adversarial examples are not bugs, they are features.
\newblock In \emph{Advances in Neural Information Processing Systems},
  125--136.

\bibitem[{Khaire and Dhanalakshmi(2019)}]{khaire2019stability}
Khaire, U.~M.; and Dhanalakshmi, R. 2019.
\newblock Stability of feature selection algorithm: A review.
\newblock \emph{Journal of King Saud University-Computer and Information
  Sciences} .

\bibitem[{Khurana, Samulowitz, and Turaga(2018)}]{khurana2018feature}
Khurana, U.; Samulowitz, H.; and Turaga, D. 2018.
\newblock Feature engineering for predictive modeling using reinforcement
  learning.
\newblock In \emph{Thirty-Second AAAI Conference on Artificial Intelligence}.

\bibitem[{Li, Schmidt, and Kolter(2019)}]{li2019adversarialc}
Li, J.; Schmidt, F.; and Kolter, Z. 2019.
\newblock Adversarial camera stickers: A physical camera-based attack on deep
  learning systems.
\newblock In \emph{International Conference on Machine Learning}.

\bibitem[{Liu et~al.(2019)Liu, Fu, Wang, Wu, Bo, and Li}]{liu2019automating}
Liu, K.; Fu, Y.; Wang, P.; Wu, L.; Bo, R.; and Li, X. 2019.
\newblock Automating Feature Subspace Exploration via Multi-Agent Reinforcement
  Learning.
\newblock In \emph{Proceedings of the 25th ACM SIGKDD International Conference
  on Knowledge Discovery \& Data Mining}, 207--215.

\bibitem[{Madry et~al.(2017)Madry, Makelov, Schmidt, Tsipras, and
  Vladu}]{madry2017towards}
Madry, A.; Makelov, A.; Schmidt, L.; Tsipras, D.; and Vladu, A. 2017.
\newblock Towards deep learning models resistant to adversarial attacks.
\newblock \emph{arXiv preprint arXiv:1706.06083} .

\bibitem[{Mnih et~al.(2013)Mnih, Kavukcuoglu, Silver, Graves, Antonoglou,
  Wierstra, and Riedmiller}]{mnih2013playing}
Mnih, V.; Kavukcuoglu, K.; Silver, D.; Graves, A.; Antonoglou, I.; Wierstra,
  D.; and Riedmiller, M. 2013.
\newblock Playing atari with deep reinforcement learning.
\newblock \emph{arXiv preprint arXiv:1312.5602} .

\bibitem[{Nazari et~al.(2018)Nazari, Oroojlooy, Snyder, and
  Tak{\'a}c}]{nazari2018reinforcement}
Nazari, M.; Oroojlooy, A.; Snyder, L.; and Tak{\'a}c, M. 2018.
\newblock Reinforcement learning for solving the vehicle routing problem.
\newblock In \emph{Advances in Neural Information Processing Systems},
  9839--9849.

\bibitem[{Ng, Harada, and Russell(1999)}]{ng1999policy}
Ng, A.~Y.; Harada, D.; and Russell, S. 1999.
\newblock Policy invariance under reward transformations: Theory and
  application to reward shaping.
\newblock In \emph{ICML}, volume~99, 278--287.

\bibitem[{Nguyen and Sanner(2013)}]{nguyen2013algorithms}
Nguyen, T.; and Sanner, S. 2013.
\newblock Algorithms for direct 0--1 loss optimization in binary
  classification.
\newblock In \emph{International Conference on Machine Learning}, 1085--1093.

\bibitem[{Peng, Long, and Ding(2005)}]{peng2005feature}
Peng, H.; Long, F.; and Ding, C. 2005.
\newblock Feature selection based on mutual information criteria of
  max-dependency, max-relevance, and min-redundancy.
\newblock \emph{IEEE Transactions on pattern analysis and machine intelligence}
  27(8): 1226.

\bibitem[{Saeys, Abeel, and Van~de Peer(2008)}]{saeys2008robust}
Saeys, Y.; Abeel, T.; and Van~de Peer, Y. 2008.
\newblock Robust feature selection using ensemble feature selection techniques.
\newblock In \emph{Joint European Conference on Machine Learning and Knowledge
  Discovery in Databases}, 313--325. Springer.

\bibitem[{Sundararajan, Taly, and Yan(2017)}]{sundararajan2017axiomatic}
Sundararajan, M.; Taly, A.; and Yan, Q. 2017.
\newblock Axiomatic attribution for deep networks.
\newblock In \emph{Proceedings of the 34th International Conference on Machine
  Learning-Volume 70}, 3319--3328. JMLR. org.

\bibitem[{Tong et~al.(2019)Tong, Li, Hajaj, Xiao, Zhang, and
  Vorobeychik}]{tong2019improving}
Tong, L.; Li, B.; Hajaj, C.; Xiao, C.; Zhang, N.; and Vorobeychik, Y. 2019.
\newblock Improving Robustness of $\{$ML$\}$ Classifiers against Realizable
  Evasion Attacks Using Conserved Features.
\newblock In \emph{28th $\{$USENIX$\}$ Security Symposium ($\{$USENIX$\}$
  Security 19)}, 285--302.

\bibitem[{Tsipras et~al.(2018)Tsipras, Santurkar, Engstrom, Turner, and
  Madry}]{tsipras2018robustness}
Tsipras, D.; Santurkar, S.; Engstrom, L.; Turner, A.; and Madry, A. 2018.
\newblock Robustness may be at odds with accuracy.
\newblock \emph{arXiv preprint arXiv:1805.12152} .

\bibitem[{Weng et~al.(2018)Weng, Zhang, Chen, Song, Hsieh, Boning, Dhillon, and
  Daniel}]{weng2018towards}
Weng, T.-W.; Zhang, H.; Chen, H.; Song, Z.; Hsieh, C.-J.; Boning, D.; Dhillon,
  I.~S.; and Daniel, L. 2018.
\newblock Towards fast computation of certified robustness for relu networks.
\newblock \emph{arXiv preprint arXiv:1804.09699} .

\bibitem[{Xu, Evans, and Qi(2017)}]{xu2017feature}
Xu, W.; Evans, D.; and Qi, Y. 2017.
\newblock Feature squeezing: Detecting adversarial examples in deep neural
  networks.
\newblock \emph{arXiv preprint arXiv:1704.01155} .

\bibitem[{Xue, Fu, and Zhang(2014)}]{xue2014multi}
Xue, B.; Fu, W.; and Zhang, M. 2014.
\newblock Multi-objective feature selection in classification: a differential
  evolution approach.
\newblock In \emph{Asia-Pacific Conference on Simulated Evolution and
  Learning}, 516--528. Springer.

\bibitem[{Xue et~al.(2019)Xue, Yan, Yan, Chu, Hu, and
  Lin}]{xue2019transferable}
Xue, C.; Yan, J.; Yan, R.; Chu, S.~M.; Hu, Y.; and Lin, Y. 2019.
\newblock Transferable automl by model sharing over grouped datasets.
\newblock In \emph{Proceedings of the IEEE Conference on Computer Vision and
  Pattern Recognition}, 9002--9011.

\bibitem[{Xue, Xie, and Roshan(2020)}]{xue2020robust}
Xue, Y.; Xie, M.; and Roshan, U. 2020.
\newblock Robust binary classification with the 01 loss.
\newblock \emph{arXiv:2002.03444} .

\bibitem[{Zhai et~al.(2020)Zhai, Dan, He, Zhang, Gong, Ravikumar, Hsieh, and
  Wang}]{zhai2020macer}
Zhai, R.; Dan, C.; He, D.; Zhang, H.; Gong, B.; Ravikumar, P.; Hsieh, C.-J.;
  and Wang, L. 2020.
\newblock MACER: Attack-free and Scalable Robust Training via Maximizing
  Certified Radius.
\newblock \emph{arXiv preprint arXiv:2001.02378} .

\bibitem[{Zhang et~al.(2019)Zhang, Yu, Jiao, Xing, Ghaoui, and
  Jordan}]{zhang2019theoretically}
Zhang, H.; Yu, Y.; Jiao, J.; Xing, E.~P.; Ghaoui, L.~E.; and Jordan, M.~I.
  2019.
\newblock Theoretically principled trade-off between robustness and accuracy.
\newblock \emph{arXiv preprint arXiv:1901.08573} .

\end{thebibliography}

\clearpage

\appendix

\section{Evaluating Existing Feature Scoring Metrics}
\label{appendix:metrics}
In Figure~\ref{fig:robustness_leakage}, we plot the variation of performance and robustness as more features are selected using existing feature scoring metrics. As we can see, as more features are selected, the performance improves but the robustness drops. Thus, the four existing score metrics fail to scoring the features w.r.t their robustness.

\begin{figure*}[!t]
\centering     
\subfigure[Random Selection]{\label{fig:feature_num_rand}\includegraphics[width=.238\textwidth]{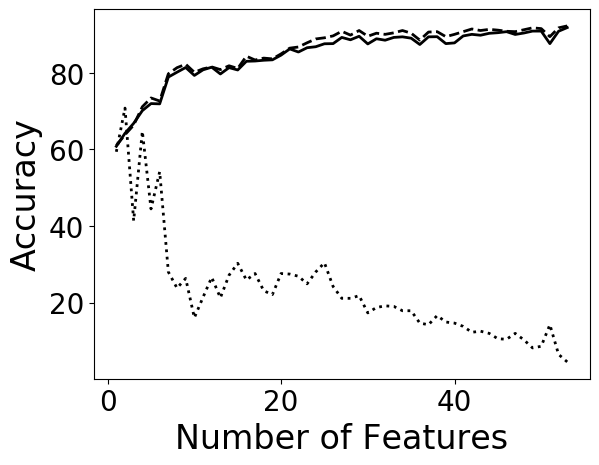}}
\subfigure[Mutual Information]{\label{fig:feature_num_mi}\includegraphics[width=.238\textwidth]{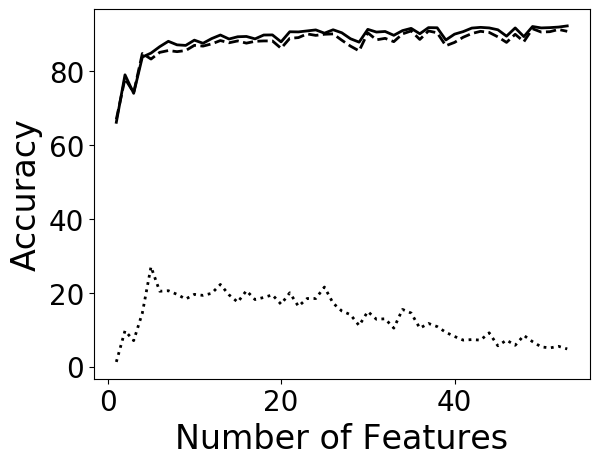}}
\subfigure[Tree Score]{\label{fig:feature_num_tree}\includegraphics[width=.238\textwidth]{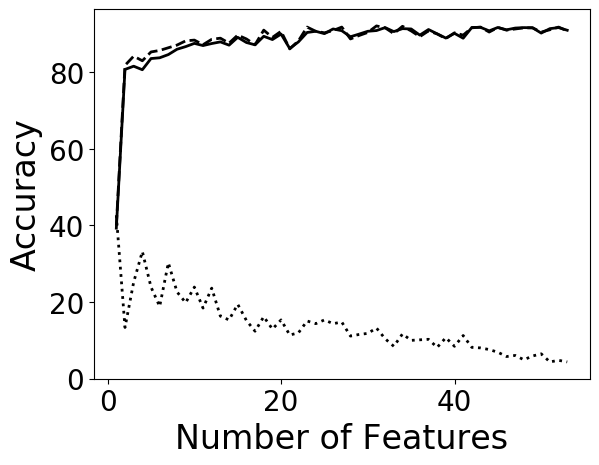}}
\subfigure[F Score]{\label{fig:feature_num_f}\includegraphics[width=.238\textwidth]{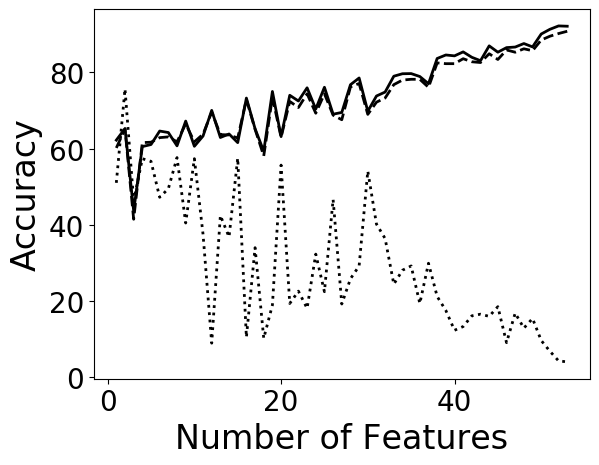}}
\caption{
\textbf{Existing Metrics Fail to Identify Robust Features on MNIST Dataset}. The performance on training set~(solid line) and test set~(dash line) increase while the robustness (dot line) decreases.}
\label{fig:robustness_leakage}
\end{figure*}

\section{Proof of Theorem 3.1}
\label{proof:theorem3.1}
\begin{theorem}
    Let us assume $\mathcal{E} = (\mathcal{S}, \mathcal{A}, \gamma, R)$ and $\mathcal{E'} = (\mathcal{S}, \mathcal{A}, \gamma, R')$ are two identical environments except reward function. If a reward function $R: \mathcal{S} \times \mathcal{A} \times \mathcal{S} \xrightarrow{} \mathbb{R}$ and a reward shaping function $F: \mathcal{S} \times \mathcal{A} \times \mathcal{S} \xrightarrow{} \mathbb{R}$ satisfy $\sum_{t=1}^{T} F(s_{t}, \pi(s_{t}), s_{t+1}) = \sum_{t=1}^{T} R(s_{t}, \pi(s_{t}), s_{t+1}) = R(s_{T}, \pi(s_{T}), s_{T+1})$ for any policy $\pi$ and any sequence length $T$, 
    then, we have $\pi_{\mathcal{E}}^{*} = \pi_{\mathcal{E'}}^{*}$ when $\gamma = 1$.
\end{theorem}

\begin{proof}
For any policy $\pi_{\mathcal{E}}$ in $\mathcal{E}$ and discount factor $\gamma = 1$. The value function is
\begin{equation}
\begin{split}
   V^{\pi_{\mathcal{E}}}(s_{0}) &= \sum_{t=1}^{T} \gamma^{t}R(s_{t}, \pi_{\mathcal{E}}(s_{t}), s_{t+1}) \\
   &= R(s_{T}^{\pi_{\mathcal{E}}}, \pi_{\mathcal{E}}(s_{T}), s_{T+1}) \\
   &= 1 - \hat{\mathcal{L}}_{\epsilon-robust}^{0-1} (g_{\bm{c}_{\pi_{\mathcal{E}}}})
\end{split}
\end{equation}

Similarly, we have the value function for policy $\pi_{\mathcal{E'}}$ in $\mathcal{E'}$:
\begin{equation}
    \begin{split}
        V^{\pi_{\mathcal{E'}}}(s_{0}) &= \sum_{t=1}^{T} R(s_{i}, \pi_{\mathcal{E'}}(s_{t}), s_{t+1}) \\
        &\ \ \ \ + \sum_{t=1}^{T} F(s_{t}, \pi_{\mathcal{E'}}(s_{t}), s_{t+1}) \\
        &= 2 \times R(s_{T}^{\pi_{\mathcal{E'}}}, \pi_{\mathcal{E'}}(s_{T}^{\pi_{\mathcal{E'}}}), s_{T+1}^{\pi_{\mathcal{E'}}})
    \end{split}
\end{equation}

Assume we have two distinct optimal policies $\pi_{\mathcal{E}}^{*}$ and $\pi_{\mathcal{E'}}^{*}$ for $\mathcal{E}$ and $\mathcal{E'}$, respectively. Since maximizing $2 \times R(s_{T}^{\pi_{\mathcal{E'}}}, \pi_{\mathcal{E'}}(s_{T}^{\pi_{\mathcal{E'}}}), s_{T+1}^{\pi_{\mathcal{E'}}})$ is equivalent to maximize $R(s_{T}^{\pi_{\mathcal{E'}}}, \pi_{\mathcal{E'}}(s_{T}^{\pi_{\mathcal{E'}}}), s_{T+1}^{\pi_{\mathcal{E'}}})$. Thus, we have $V^{\pi_{\mathcal{E'}}^{*}}(s_{0}) = R(s_{T}^{\pi_{\mathcal{E'}}^{*}}, \pi_{\mathcal{E'}}^{*}(s_{T}^{\pi_{\mathcal{E'}}^{*}}), s_{T+1}^{\pi_{\mathcal{E'}}^{*}}) > R(s_{T}^{\pi_{\mathcal{E}}^{*}}, \pi_{\mathcal{E}}^{*}(s_{T}^{\pi_{\mathcal{E}}^{*}}), s_{T+1}^{\pi_{\mathcal{E}}^{*}}) = V^{\pi_{\mathcal{E}}^{*}}(s_{0})$. However, by definition, we have $V^{\pi_{\mathcal{E}}^{*}}(s_{0}) > V^{\pi_{\mathcal{E}}}(s_{0})$ for any given $\pi_{\mathcal{E}} \in \Pi_{\mathcal{E}} - \pi_{\mathcal{E}}^{*}$. As $\mathcal{E}$ and $\mathcal{E'}$ are only different on the reward function, we have $\Pi_{\mathcal{E}} = \Pi_{\mathcal{E'}}$. Thus, we have $V^{\pi_{\mathcal{E}}^{*}}(s_{0}) > V^{\pi_{\mathcal{E'}}}(s_{0})$ for any given $\pi_{\mathcal{E'}} \in \Pi_{\mathcal{E'}} - \pi_{\mathcal{E}}^{*}$. Since $V^{\pi_{\mathcal{E'}}^{*}}(s_{0}) > V^{\pi_{\mathcal{E}}^{*}}(s_{0})$ conflicts $V^{\pi_{\mathcal{E}}^{*}}(s_{0}) > V^{\pi_{\mathcal{E'}}}(s_{0})$, we must have $\pi_{\mathcal{E}}^{*} = \pi_{\mathcal{E'}}^{*}$.
\end{proof}

\section{Theorem for Potential-based Reward Shaping Function}
Another way of proving $\pi_{\mathcal{E}}^{*} = \pi_{\mathcal{E'}}^{*}$ is to leverage the previous result of potential-based reward shaping function and use the following theorem
\begin{theorem}
    (Theorem 1 in \citealt{ng1999policy})
    Let any $\mathcal{S}$, $\mathcal{A}$, $\gamma$ and any shaping reward function $F: \mathcal{S} \times \mathcal{A} \times \mathcal{S} \xrightarrow{} \mathbb{R}$ be given. We say F is a potential-based shaping function such that for all $s \in \mathcal{S} - \{s_{0}\}$, $a \in \mathcal{A}$, $s' \in \mathcal{S}$,
    \begin{equation}
    \label{equation:reward_shaping}
        F(s, a, s') = \gamma\Phi(s') - \Phi(s),
    \end{equation}
    (where $\mathcal{S} - \{s_{0}\} = \mathcal{S}$ if $\gamma < 1$) Then, that F is a potential-based shaping function is a necessary and sufficient condition for it to guarantee consistence with the optimal policy (when learning from $\mathcal{E'} = (\mathcal{S}, \mathcal{A}, T, \gamma, R + F)$ rather than from $\mathcal{E} = (\mathcal{S}, \mathcal{A}, T, \gamma, R)$) if the following sense: 
    \begin{itemize}
    \itemsep 0in
    \item (Sufficiency) If F is a potential-based shaping function, then every optimal policy in $\mathcal{E'}$ will also be an optimal policy in $\mathcal{E}$.
    \item (Necessity) If F is not a potential-based shaping function (e.g. no such $\Phi$ exists satisfying Equation~\eqref{equation:reward_shaping}), then there exist (proper) transition function $T$ and a reward function $R$, such that no optimal policy in $\mathcal{E'}$ is optimal in $\mathcal{E}$. 
    \end{itemize}
\end{theorem}

\section{Intuitive Explanation of IG score}
\label{section:intuition_IG}
We show that Integrated Gradient (IG) is an effective method to estimate the robustness of features by measuring the local sharpness along each dimension represented by each feature. The sharpness is a good indicator of robustness because the sharp gradient is a major cause of adversarial examples. Let us consider a simple case where a dataset $D$ has two features $0, 1$ and one benign sample, $\{\bm{x} \mid \bm{x} \in \mathbb{R}^{2}\}$. After adding adversarial noise $\bm{\delta}$, we have an adversarial example $\bm{x}_{adv} = \bm{x} + \bm{\delta}$. We further assume the gradient is near $0$ in $[\bm{x}^{0}, \bm{x}_{adv}^{0}]$ while the gradient is $c$ in $[\bm{x}^{1}, \bm{x}_{adv}^{1}]$. Here, feature $0$ represents the direction of $[\bm{x}^{0}, \bm{x}_{adv}^{0}]$ and feature $1$ represents the direction of $[\bm{x}^{1}, \bm{x}_{adv}^{1}]$. Formally, we have $\frac{\partial f_{\bm{w}}^{y}(\bm{x})}{\partial \bm{x}^{0}} \approx 0$ and $\frac{\partial f_{\bm{w}}^{y}(\bm{x})}{\partial \bm{x}^{1}} = c$. Substituting the gradient in equation~\eqref{ig} with $\frac{\partial f_{\bm{w}}^{y}(\bm{x})}{\partial \bm{x}^{0}}$ and $\frac{\partial f_{\bm{w}}^{y}(x)}{\partial \bm{x}^{1}}$, we have $\bm{IG}^{0}(\bm{x}) \approx 0$ and $\bm{IG}^{1}(\bm{x}) \approx c\times(\bm{x}^{1}-\bm{x}_{adv}^{1})$. Obviously, feature $1$ will receive more attribution than feature $0$. On the other hand, it is indeed the sharper gradient $c$ along feature $1$ together with the perturbation $(\bm{x}^{1}-\bm{x}_{adv}^{1})$ makes $\bm{x}_{adv}$ an adversarial example.

\section{RL Actions}
\label{section:rl_actions}
We report the actions of the RL agent in a single episode in Table~\ref{table:rl_action}, from which we choose the subset of features for further evaluation. As Table~\ref{table:rl_action} shows, the RL agent relies on different metrics to achieve the objective across the dataset, which demonstrates the RL agent is capable for identifying the proper way of ensembling the feature scoring metrics. Table~\ref{table:rl_action} also suggests that the learned RL agent is not transferable between the datasets. It is because the 4 datasets we are using are quite different from each other. The AutoML transferability over grouped datasets or continuously changing datasets has been studied on related works \cite{liu2019automating, xue2019transferable}.

\begin{table}[t]
\caption{Action Counts of RL Agent}
\label{table:rl_action}
\begin{center}
\begin{small}
\begin{threeparttable}
\begin{sc}
\begin{tabular}{lccccccr}
\toprule
Dataset & MI & Tree & F & MI\textsubscript{IG} & Tree\textsubscript{IG} & F\textsubscript{IG} \\
\midrule
SpamBase    & 0 & 0 & 2 & 1 & 2 & 0 \\
Isolet & 8 & 28 & 1 & 2 & 3 & 0 \\
MNIST    & 2 & 20 & 0 & 0 & 49 & 0 \\
CIFAR    & 6 & 1 & 12 & 2 & 3 & 17 \\
\bottomrule
\end{tabular}
\end{sc}
\begin{tablenotes}
\item [*] The metrics with IG subscripts denote the combined metrics defined in Section~\ref{robust_metric}.
\end{tablenotes}
\end{threeparttable}
\end{small}
\end{center}
\end{table}

\section{Experiment Setup and Details}
\label{section:experiment_setup}
\paragraph{Experimental Environment} We performed our evaluation on a single machine with an Nvidia RTX 2070 GPU and an Intel Core i9 CPU as well as 32GB memory and 1TB hard disk.

\paragraph{ML Model} We used a 2-layer fully connected neural network $f_{\bm{w}}$ with a hidden layer of size 300 as the ML model. The other dimensions of $\bm{w}$ are determined by the feature selection vector $\bm{c}$ and the label space $\mathcal{Y}$.

\paragraph{Q-Network} We used a 2-layer fully connected neural network as the Q-network $Q_{\theta}$, whose input size was 15, the hidden size was 128, output size was 6.

\paragraph{Optimizer} We used a default Adam optimizer with $\beta1=0.9$, $\beta2=0.999$ through out the experiment.

\paragraph{RL Agent} We tried [2000, 3000, 4000] for number of steps and [1000, 2000, 3000] for the $\epsilon$-decay ($\epsilon$ decays over [1000, 2000, 3000] steps). The other settings are the same as \cite{mnih2013playing}. Finally, we trained our RL model for 4000 steps with learning rate = 0.01, batch size = 64, discount factor $\gamma$ = 1 and $\epsilon$ = 0.1 ($\epsilon$-greedy). The training started after the first 100 steps and the target network updated every 1000 steps.

\paragraph{Computational Cost}
The framework need 1-4 hours to complete, depends on the size of the dataset, which is slower than stable feature selection on small dataset like SpamBase and Isolet. However, it's faster on MNIST and CIFAR.

\paragraph{Choose $\epsilon$ for Adversary}
We start the adversary attack with $\epsilon = 8/255$. On Isolet and MNIST, $\epsilon = 8/255$ is not sufficient to significantly decrease the accuracy. Thus, we start again with $\epsilon = 1/10$ and gradually increase it by $1/10$

\paragraph{Select the Number of Features}
The optimal number of features for each method is close to the number of features that Robusta chooses. As is shown in Figure~\ref{fig:num_feature}, LASSO achieves good performance and robustness with 71 features on MNIST and 4 features on LASSO. Thus, to decide the optimal number of features for each method, we simply perform grid search, with granularity 5, around the number of features that Robusta chooses. For stable feature selection, we tried to tune the hyper-parameter but the algorithm was unable to return a small set of selected features. Thus, we base the evaluation on the smallest set of features we can get, which leads to maximum robustness. The full details are reported in Table~\ref{table:num_feature}.

\begin{table*}[!t]
\caption{ML Model Hyper-parameters for Feature Evaluation}
\label{table:num_feature}
\begin{center}
\begin{small}
\begin{sc}
\begin{tabular}{lccccc}
\toprule
Data set & Stable & Mutual Info & LASSO & Concrete & Robusta \\
\midrule
SpamBase & 22 & 4 & 4 & 4 & 4 \\
Isolet & 61 & 40 & 45 & 50  & 40 \\
MNIST & / & 71 & 71 & 66 & 71 \\
CIFAR & / & 30 & 25 & 30 & 30 \\
\bottomrule
\end{tabular}
\end{sc}
\end{small}
\end{center}
\end{table*}

\begin{figure*}
\centering     
\subfigure[]{\includegraphics[width=.33\textwidth]{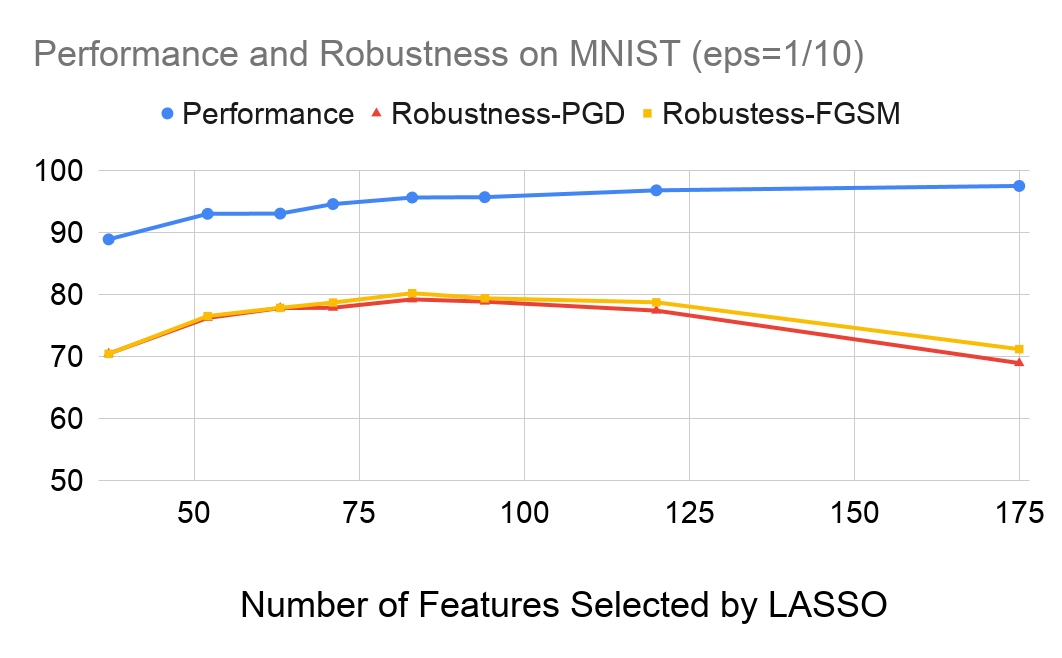}}
\subfigure[]{\includegraphics[width=.33\textwidth]{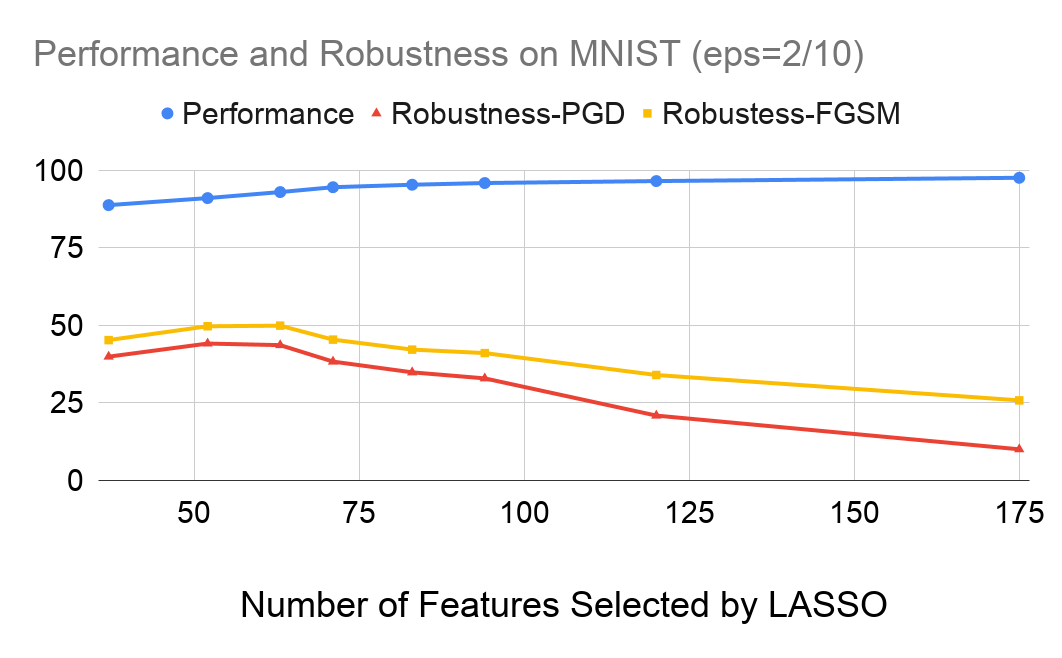}}
\subfigure[]{\includegraphics[width=.33\textwidth]{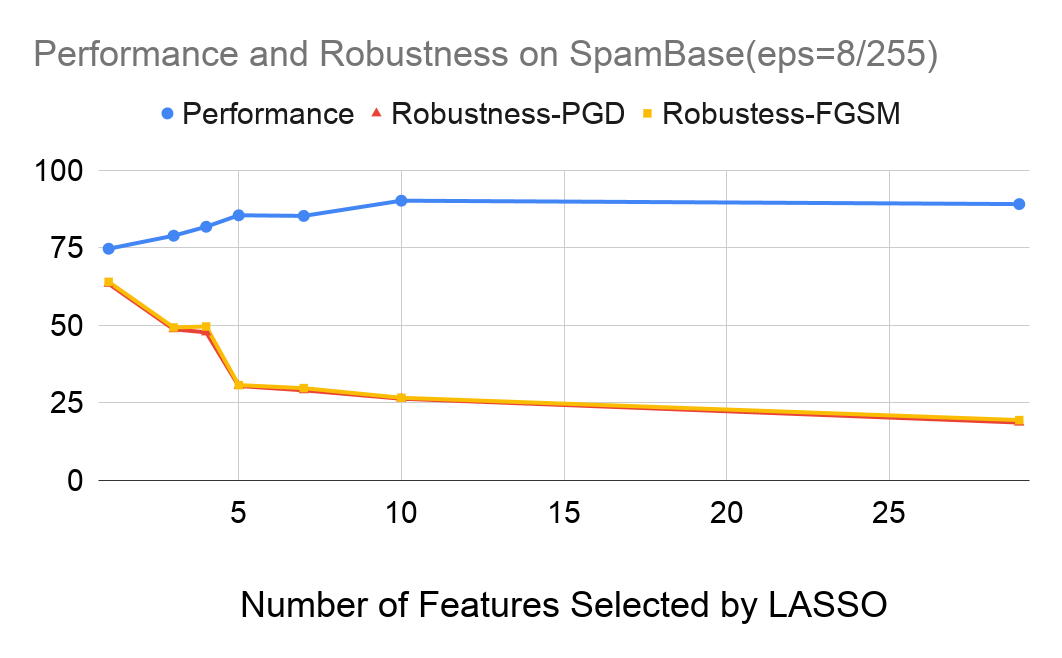}}
\caption{Performance and Robustness with Different Number of Features}
\label{fig:num_feature}
\end{figure*}

\paragraph{Feature Scoring Metric Computation}
All the feature scores can be directly calculated by the Scikit-learn\footnote{\url{https://scikit-learn.org/stable/modules/generated/sklearn.feature_selection.mutual_info_classif.html#sklearn.feature_selection.mutual_info_classif}}\footnote{\url{https://scikit-learn.org/stable/modules/generated/sklearn.ensemble.ExtraTreesClassifier.html}}\footnote{\url{https://scikit-learn.org/stable/modules/generated/sklearn.feature_selection.f_classif.html#sklearn.feature_selection.f_classif}} and Captum\footnote{\url{https://captum.ai/docs/extension/integrated_gradients}} libraries.

\paragraph{IG Reuse}
Computing the IG attribution is expensive, especially when we need to compute the IG attribution at each step in RL. To alleviate this problem, we only compute one IG attribution for an ML model $f_{\bm{w}}$, which is trained with full features, and reuse it at all the steps. Using the same IG attribution across all steps also yield good result. The reason is that the robustness of \textit{features} does not heavily depend on the ML model \textit{parameters} as~\cite{ilyas2019adversarial} shows. Figure~\ref{fig:ig_heatmap_different_net} visualizes the similar IG attributions for three neural networks with different parameters and attack methods.

\begin{figure}[!t]
\centering     
\subfigure[Net 0, PGD]{\label{fig:ig_heatmap_010}\includegraphics[width=.153\textwidth]{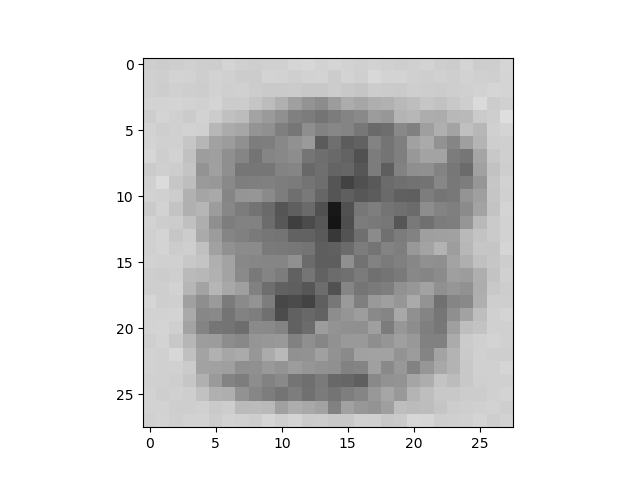}}
\subfigure[Net 1, PGD]{\label{fig:ig_heatmap_011}\includegraphics[width=.153\textwidth]{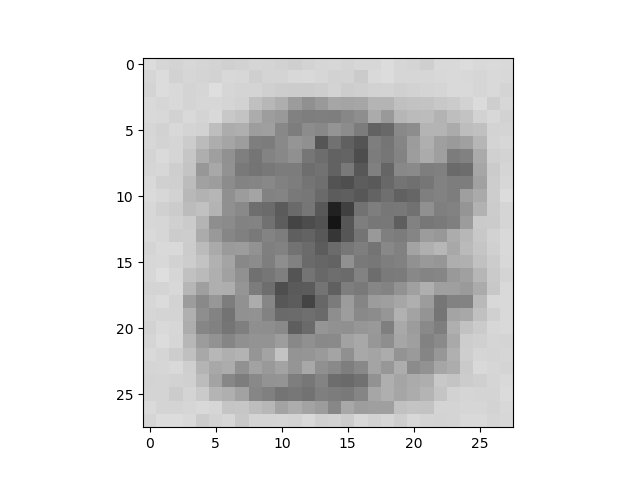}}
\subfigure[Net 1, FGSM]{\label{fig:ig_heatmap_111}\includegraphics[width=.153\textwidth]{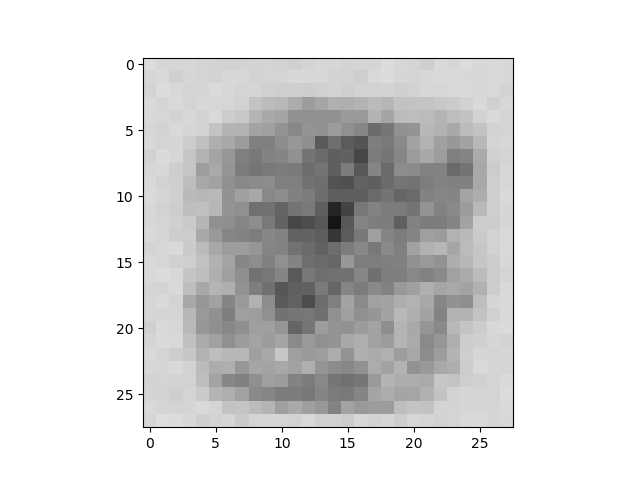}}
\caption{\textbf{Heat-Map of Feature Attribution from Two Neural Networks with Different Parameters under Different Attacks on MNIST Dataset} The perturbation size $\epsilon$ is  0.1.}
\label{fig:ig_heatmap_different_net}
\end{figure}

\paragraph{Feature Evaluation}
We trained a 2-layer neural network on each of the feature sets in order to evaluate the quality of the selected features. We report the training hyper-parameters in Table~\ref{hyper_parameters}. For the learning rate, we performed a grid search over [0.01, 0.1] with step size 0.01. For the training epoch, we performed a grid search over [10, 50] with step size 10. We selected the hyper-parameters where the performance on training and validation set didn't diverse much.

\begin{table}[!t]
\caption{ML Model Hyper-parameters for Feature Evaluation}
\label{hyper_parameters}
\begin{center}
\begin{small}
\begin{sc}
\begin{tabular}{lccr}
\toprule
Data set & Learning Rate & Epoch \\
\midrule
SpamBase & 0.05 & 30 \\
Isolet & 0.01 & 20 \\
MNIST & 0.05 & 20 \\
CIFAR & 0.05 & 10 \\
\bottomrule
\end{tabular}
\end{sc}
\end{small}
\end{center}
\end{table}

\paragraph{Datasets}
We performed experiments on four real-world datasets which contain tabular data, acoustic data, and image data:

\textit{SpamBase}: consists of 5000 spam and non-spam e-mails. 
The e-mails are represented in a tabular way using word and character frequencies, so the features are corresponding continuous values. It has 57 continuous features in total. But we only select 54 of them because the values of the remaining 3 are unbounded, while the values of the former 54 features are bounded between 0 and 100. We further scale the features to $[0, 1]$. We choose this dataset as spam filtering is a common task where attackers can get benefits.

\textit{Isolet}: consists of 5000 pre-processed speech data samples of people speaking the names of the letters in the English alphabet. This dataset is widely used as a benchmark in the feature selection literature. Each feature is one of the 617 quantities in the range $[0, 1]$. We choose this dataset because acoustic data is vulnerable to adversarial attacks.

\textit{MNIST}: consists of 60,000 training and 10,000 test images, which are 28-by-28 gray-scale images of hand-written digits. We choose these datasets because they are widely known in the machine learning community. Although there are other image datasets, the objects in each image in MNIST are centered, which means we can meaningfully treat every 784 pixels in the image as a separate feature.

\textit{CIFAR10}: consists of 50,000 training and 10,000 test examples with 10 classes, which are 32x32 colour images. Different from MNIST, where each image is centered, objects in CIFAR10 may appear in different parts of images. As a result, each pixel in the image is not a good candidate for a feature. Instead, we use pre-trained ResNet-18 networks to extract hidden representations from the images and treat the extracted representations as features.

\paragraph{Data Sample Allocation}
For the Robusta training, we partitioned the training set in each dataset into RL training, RL validation, and RL test set. At each step, We used the RL training set to train an ML model with the current subset of features. We then measured the features' performance on the RL validation set. We also made a corrupted dataset from the RL validation set, which contained corrupted samples with Gaussian noise. The performance on the RL validation set and the robustness on the corrupted validation set were used to compute the reward. We made an adversarial test set by performing adversarial attacks on the RL test set. We recorded the performance on the RL test set and robustness on the adversarial test set and reported them in Figure~\ref{fig:rl_perf}. For the feature subset evaluation, we used the standard training/test sets. We trained the model on the training set and evaluated the model on the test set. Hyper-parameter tuning is directly done using the training set. Through out the experiment, the data that is used to train\&evaluate the RL agent and the data that used to evaluate the selected features are hold exclusive.

\paragraph{Tips and Tricks for Training the RL Agent}
We describe three tricks, which is associated with the reward function $R$ and the reward shaping function $F$, we used in the experiment and their application scenarios. (1) If the RL agent only selects features to improve either performance or robustness while ignores the other one, assign different weights to the terms in the reward function $R$ and the reward shaping function $F$. More specifically, tune $\lambda_{nat}$ and $\lambda_{Gaussian}$ in $\lambda_{nat} \cdot \hat{\mathcal{L}}_{nat}^{0-1}(g_{\bm{c}}) + \lambda_{Gaussian} \cdot \hat{\mathcal{L}}_{\epsilon-Gaussian}^{0-1}(g_{\bm{c}})$. (2) If the first trick does not work, try to add a clipping function to $\hat{\mathcal{L}}_{nat}^{0-1}(g_{\bm{c}_{s}}) - \hat{\mathcal{L}}_{nat}^{0-1}(g_{\bm{c}_{s'}})$ or $\hat{\mathcal{L}}_{\epsilon-Gaussian}^{0-1}(g_{\bm{c}_{s}}) - \hat{\mathcal{L}}_{\epsilon-Gaussian}^{0-1}(g_{\bm{c}_{s'}})$ in the reward shaping function $F$ and make it saturate at large value. This trick is not recommended because it may break Theorem~\ref{theorem:3.1}. (3) If the RL agent only selects features from one of the scoring metrics and does not produce desirable selection results, the potential reason is that the first few actions get large rewards while the following actions get small rewards because the accuracy improves slower when many features have been selected. A possible fix is to multiply the reward shaping function $F$ with $\mathrm{log}(m)$ where $m$ is the number of features that have been selected. This is also not recommended due to the same reason as the second trick. However, try these tricks if the RL agent consistently fails to produce a desirable result.

\section{Evaluating the Effectiveness of Robusta Framework Under FGSM Attack}
\label{section:FGSM_robustness}

Similar to the experiment with PGD attacks in section~\ref{section:rl_eval}, the features selected by our framework is consistently more robust against FGSM attacks as Table~\ref{table:feature_robustness_fgsm} shows.

\begin{table*}[!t]
\caption{Performance (accuracy on benign samples) of the ML Model using selected features}
\label{feature_performance}
\begin{center}
\begin{small}
\begin{threeparttable}
\begin{sc}
\begin{tabular}{lcccccc}
\toprule
Data set ($\epsilon$) & Stable & Mutual Info & LASSO & Concrete & Robusta \\
\midrule
Spam ($8/255$) & \textbf{91.7 $\pm$ 1.98\%} & 85.10\% $\pm$ 2.35\% & 80.06\% $\pm$ 0.87\% & 80.36\% $\pm$ 1.85\% & 77.27\% $\pm$ 0.55\% \\

Isolet ($1/10$) & \textbf{91.7 $\pm$ 1.65\%} & 59.50\% $\pm$ 1.79\% & 76.65\% $\pm$ 0.39\% & 81.54\% $\pm$ 0.22\% & 81.99\% $\pm$ 0.19\% \\

MNIST ($1/10$) & / & 55.02\% $\pm$ 0.47\% & 94.55\% $\pm$ 0.02\% & 97.21\% $\pm$ 0.08\% & 95.76\% $\pm$ 0.11\% \\

MNIST ($2/10$) & / & 54.65\% $\pm$ 0.41\% & 94.54\% $\pm$ 0.07\% & 97.24\% $\pm$ 0.09\% & 95.71\% $\pm$ 0.21\% \\

MNIST ($3/10$) & / & 54.60\% $\pm$ 0.80\% & 94.58\% $\pm$ 0.18\% & 97.22\% $\pm$ 0.07\% & 95.68\% $\pm$ 0.19\% \\

CIFAR ($8/255$) & / & 94.13\% $\pm$ 0.12\% & 94.43\% $\pm$ 0.08\% & 94.44\% $\pm$ 0.02\% & 90.92\% $\pm$ 0.09\% \\
\bottomrule
\end{tabular}
\end{sc}
\begin{tablenotes}
\item [*]We bold the numbers if the best method outperforms all the others by 3\%.
\end{tablenotes}
\end{threeparttable}
\end{small}
\end{center}
\end{table*}

\begin{table*}[!t]
\caption{Robustness (accuracy on adversarial examples) of the ML model using selected features under PGD attack}
\label{feature_robustness_pgd}
\begin{center}
\begin{small}
\begin{threeparttable}
\begin{sc}
\begin{tabular}{lcccccc}
\toprule
Data set ($\epsilon$) & Stable & Mutual Info & LASSO & Concrete & Robusta \\
\midrule
Spam ($8/255$) & 18.10\% $\pm$ 0.67\% & 14.03\% $\pm$ 1.55\% & 55.36\% $\pm$ 2.43\% & 49.73\% $\pm$ 0.75\% & \textbf{68.03\% $\pm$ 1.53\%} \\

Isolet ($1/10$) & 25.98\% $\pm$ 1.03\% & 23.28\% $\pm$ 1.38\% & 42.74\% $\pm$ 0.55\% & 24.13\% $\pm$ 0.59\% & \textbf{48.02\% $\pm$ 1.07\%} \\

MNIST ($1/10$) & / & 27.56\% $\pm$ 0.46\% & 77.82\% $\pm$ 0.12\% & 77.93\% $\pm$ 1.35\% & \textbf{83.19\% $\pm$ 0.35\%} \\

MNIST ($2/10$) & / & 16.44\% $\pm$ 0.40\% & 38.27\% $\pm$ 1.77\% & 27.10\% $\pm$ 1.24\% & \textbf{44.87\% $\pm$ 1.74\%}  \\

MNIST ($3/10$) & / & 10.56\% $\pm$ 1.40\% & 14.14\% $\pm$ 0.46\% & 4.67\% $\pm$ 0.91\% & \textbf{18.11\% $\pm$ 1.76\%} \\

CIFAR ($8/255$) & / & 1.86\% $\pm$ 0.11\% & 7.25\% $\pm$ 0.18\% & 14.29\% $\pm$ 0.28\% & \textbf{36.74\% $\pm$ 0.46\%} \\
\bottomrule
\end{tabular}
\end{sc}
\begin{tablenotes}
\item [*]We bold the numbers if the best method outperforms all the others by 3\%.
\end{tablenotes}
\end{threeparttable}
\end{small}
\end{center}
\end{table*}

\begin{table*}[!t]
\caption{Robustness of the ML Model using Selected Features under FGSM attack}
\label{table:feature_robustness_fgsm}
\begin{center}
\begin{small}
\begin{threeparttable}
\begin{sc}
\begin{tabular}{lccccr}
\toprule
Data set & Mutual Info & LASSO & Concrete & Robusta \\
\midrule
SpamBase($\epsilon = 8/255$) & 14.11\% $\pm$ 1.94\% & 57.00\% $\pm$ 1.92\% & 51.13\% $\pm$ 1.66\% & \textbf{69.10\% $\pm$ 1.04\%} \\
Isolet($\epsilon = 1/10$) & 23.49\% $\pm$ 1.07\% & 42.50\% $\pm$ 0.36\% & 26.12\% $\pm$ 0.68\% & \textbf{47.18\% $\pm$ 1.32\%} \\
MNIST($\epsilon = 1/10$) & 30.48\% $\pm$ 0.64\% & 78.66\% $\pm$ 0.23\% & 78.69\% $\pm$ 0.68\% & \textbf{83.37\% $\pm$ 0.41\%} \\
MNIST($\epsilon = 2/10$) & 20.57\% $\pm$ 0.14\% & 45.36\% $\pm$ 1.77\% & 38.14\% $\pm$ 0.79\% & \textbf{49.92\% $\pm$ 1.05\%}  \\
MNIST($\epsilon = 3/10$) & 14.50\% $\pm$ 0.87\% & 25.26\% $\pm$ 1.24\% & 18.45\% $\pm$ 0.92\% & \textbf{29.00\% $\pm$ 2.12\%} \\
CIFAR($\epsilon = 8/255$) & 49.08\% $\pm$ 0.11\% & 53.30\% $\pm$ 0.20\% & 56.43\% $\pm$ 0.05\% & \textbf{63.20\% $\pm$ 0.19\%} \\
\bottomrule
\end{tabular}
\end{sc}
\begin{tablenotes}
\item [*]We bold the numbers if the best method outperforms all the others by 3\%.
\end{tablenotes}
\end{threeparttable}
\end{small}
\end{center}
\end{table*}

\end{document}